\documentclass[10pt,twocolumn,letterpaper]{article}

\usepackage{cvpr} 
\usepackage{CJK}
\usepackage{multirow}
\usepackage{multicol}
\usepackage{makecell}
\usepackage{color}
\usepackage{xcolor}
\usepackage{colortbl}
\usepackage{bm}
\usepackage{amsmath}
\usepackage{extarrows}
\usepackage{float}
\usepackage{amsmath, amssymb, amsthm}
\usepackage[most]{tcolorbox} % 引入最强大的彩色框包
\usepackage{xcolor}
\usepackage{mdframed} % 用于加框 (可选)
\usepackage{bbm}
\usepackage{algorithm}
\usepackage{algorithmic}
\usepackage{amsmath}
\definecolor{theoremcolor}{RGB}{0, 0, 139} % DarkBlue
\definecolor{defcolor}{RGB}{0, 100, 0}     % DarkGreen
\definecolor{thmblue}{RGB}{0, 0, 100}      % 深蓝色 (用于定理/引理)
\definecolor{thmback}{RGB}{240, 245, 255}  % 极淡的蓝色背景
\definecolor{defgreen}{RGB}{0, 60, 0}      % 深绿色 (用于定义)
\definecolor{defback}{RGB}{240, 255, 240}  % 极淡的绿色背景
\definecolor{obsorange}{RGB}{100, 40, 0}   % 深橙褐色 (用于观察)
\definecolor{obsback}{RGB}{255, 250, 240}  % 极淡的橙色背景

\newtheoremstyle{cvprthm}
  {3pt}   % Space above
  {3pt}   % Space below
  {\itshape} % Body font
  {}      % Indent amount
  {\bfseries\color{theoremcolor}} % Theorem head font
  {.}     % Punctuation after theorem head
  {.5em}  % Space after theorem head
  {}      % Theorem head spec (can be left empty, meaning 'normal')

\newtheoremstyle{cvprdef}
  {3pt}   % Space above
  {3pt}   % Space below
  {\rmfamily} % Body font (Definition 通常用正体)
  {}      % Indent amount
  {\bfseries\color{defcolor}} % Theorem head font
  {.}     % Punctuation after theorem head
  {.5em}  % Space after theorem head
  {}      % Theorem head spec

\theoremstyle{cvprthm}
\newtheorem{theorem}{Theorem}

\theoremstyle{cvprdef}

\definecolor{cvprblue}{rgb}{0.21,0.49,0.74}
\usepackage[pagebackref,breaklinks,colorlinks,citecolor=brown]{hyperref}

\def\paperID{*****} % *** Enter the Paper ID here
\def\confName{CVPR}
\def\confYear{2026}
\definecolor{supcolor}{HTML}{D2691E}

\newcommand{\statement}[1]{\noindent\textbf{#1}}

\newcommand\blfootnote[1]{%
  \begingroup
  \renewcommand\thefootnote{}\footnote{#1}%
  \addtocounter{footnote}{-1}%
  \endgroup
}
\title{Air-Know: Arbiter-Calibrated Knowledge-Internalizing Robust Network for Composed Image Retrieval}

\author{Zhiheng Fu~~~~Yupeng Hu$^{*}$\blfootnote{~corresponding authors}~~~~Qianyun Yang~~~~Shiqi Zhang~~~~Zhiwei Chen~~~~Zixu Li \vspace{2mm}\\
Shandong University\hspace{1.5cm}\\
{\tt\small \{fuzhiheng8,zivczw,lizixu.cs\}@gmail.com;} \\ 
{\tt\small \{qianyunyang,zhangshiqi\}@mail.sdu.edu.cn; huyupeng@sdu.edu.cn}\\
{\tt\small \url{https://github.com/ZhihFu/Air-Know/}}
}

\begin{document}
% \begin{CJK}{UTF8}{gbsn}

 \twocolumn[{%
\maketitle
\begin{figure}[H]

\hsize=\textwidth 
\centering
\includegraphics[width=\textwidth]{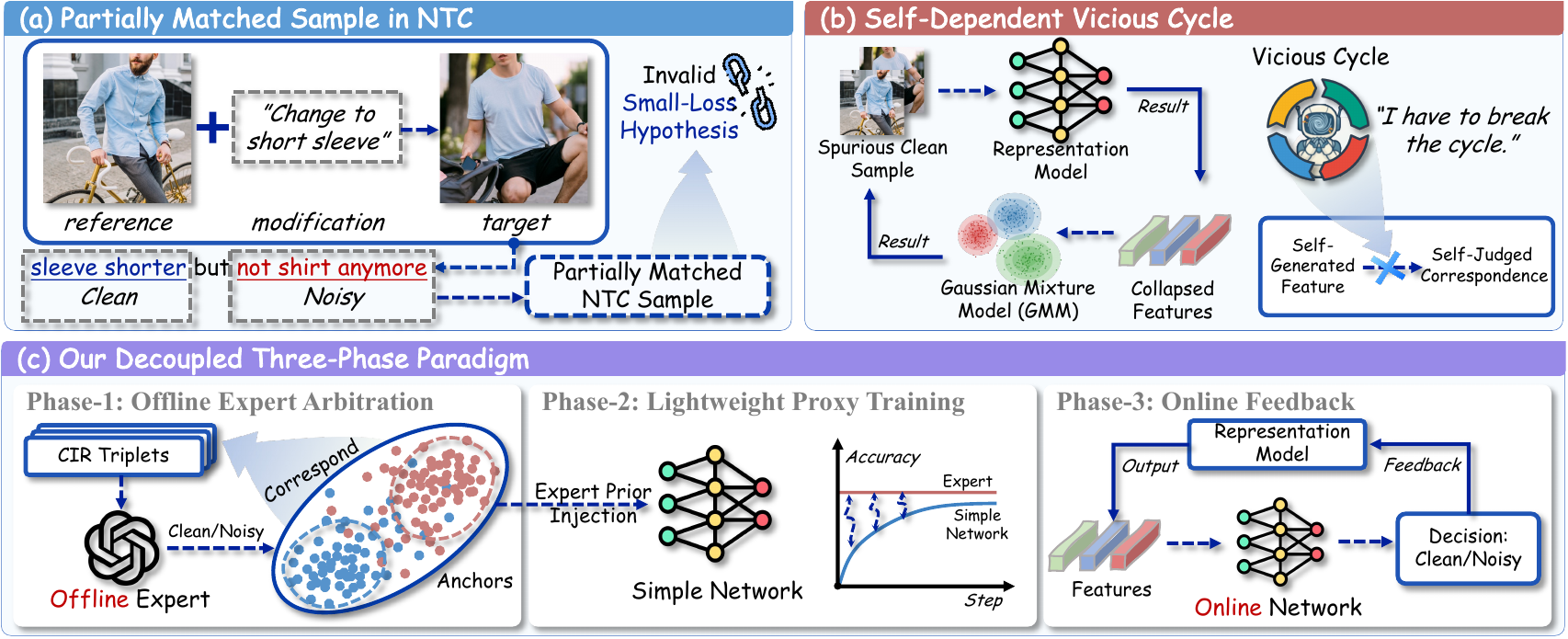}
\vspace{-20pt}
\caption{(a) illustrates the semantic ambiguity of noise in NTC.
(b) illustrates the vicious cycle of self\mbox{-}dependency caused by unreliable noise determination. (c) introduces our proposed ``Expert\mbox{-}Proxy\mbox{-}Diversion'' three-phase learning framework. Figure best viewed in color.}
\label{fig:intro}
\end{figure}
 }]
\blfootnote{$^*$ Corresponding Author: Yupeng Hu}

\begin{abstract}
Composed Image Retrieval (CIR) has attracted significant attention due to its flexible multimodal query method, yet its development is severely constrained by the Noisy Triplet Correspondence (NTC) problem. Most existing robust learning methods rely on the ``small loss hypothesis'', but the unique semantic ambiguity in NTC, such as ``partial matching'', invalidates this assumption, leading to unreliable noise identification. This entraps the model in a self dependent vicious cycle where the \textbf{learner} is intertwined with the \textbf{arbiter}, ultimately causing catastrophic ``representation pollution''. To address this critical challenge, we propose a novel ``Expert\mbox{-}Proxy\mbox{-}Diversion'' decoupling paradigm, named \textbf{Air\mbox{-}Know} (ArbIteR calibrated Knowledge iNternalizing rObust netWork). Air\mbox{-}Know incorporates three core modules: (1) \textit{External Prior Arbitration (EPA)}, which utilizes Multimodal Large Language Models (MLLMs) as an offline expert to construct a high precision anchor dataset; (2) \textit{Expert Knowledge Internalization (EKI)}, which efficiently guides a lightweight proxy ``arbiter'' to internalize the expert's discriminative logic; (3) \textit{Dual Stream Reconciliation (DSR)}, which leverages the EKI's matching confidence to divert the training data, achieving a clean alignment stream and a representation feedback reconciliation stream. Extensive experiments on multiple CIR benchmark datasets demonstrate that Air\mbox{-}Know significantly outperforms existing SOTA methods under the NTC setting, while also showing strong competitiveness in traditional CIR.
\end{abstract}
    
\section{Introduction}
In recent years, Composed Image Retrieval (CIR)~\cite{TME,habit,MELT} has garnered significant attention in the field of information retrieval~\cite{intent,xie2025chat,sprc,STABLE,liu2025queries,QuRe,liu2023retrieval,xie2026conquer,tian2025core,xu2025hdnet,jiang2025transforming,bi2025llava,lu2026reassessing,jiang2024prior}, multimodal learning~\cite{li2024optimizing,lin2026scientific,song7,zhang2026towards,yu2025cotextor,xie2026delving,Long_2026,lin2025se,song3,liu2025fusion,song4,ERASE,tian2025open,bi2025prismselfpruningintrinsicselection} and computer vision~\cite{song1,yu2025yielding,xiao2025visual,xiao2025prompt,li2024videocogqa,song2,zhong2026collaborativemultiagentscriptsgeneration,xie2026hvd,jiang2025self,song6,wang2025ascd,bi2025cot,sun2023hierarchical,yu2026dinov3,10.1145/3746027.3755817}, owing to its flexible paradigm that utilizes a reference image and a modification text to retrieve a target image. Despite significant progress, practicality of CIR is constrained by a critical bottleneck: noisy correspondence within the training data. Due to the high cost and inherent subjectivity of triplet annotation, coupled with the hallucination issues prevalent in annotations by large models, datasets are replete with erroneous matches or matches exhibiting semantic ambiguity. To address this issue, Li et al.~\cite{TME} explicitly formulated this challenge as the Noisy Triplet Correspondence (NTC) problem, focusing their research on training a robust CIR model capable of effectively resisting noise interference.

To address NTC problem, recent works (e.g.,TME~\cite{TME}) have begun to draw upon insights from existing robust methods~\cite{recon,sun2024robust,li2023cross,qin2023cross,yuan2025prototype,li2024incomplete} designed for cross-modal noisy correspondence learning (NCL). These methods typically rely on the small-loss hypothesis~\cite{unsupervised,closer}, which posits that clean samples generate low losses while the noisy yields high losses, and utilize this premise (for instance, via GMM) to partition the samples. However, given that the ``noise'' within NTC exhibits unique semantic ambiguity, the above methods are difficult to directly apply to NTC scenarios.

Specifically, traditional Noisy Label Learning (NLL)~\cite{Feng2023MaskCon,Feng2024NoiseBox} typically deals with ``randomly flipped labels'', while standard Noisy Correspondence Learning (NCL)~\cite{noisy-cross} addresses completely mismatched cross-modal pairs. In contrast, noise in NTC~\cite{TME} uniquely consists of ``partial matches'' characterized by semantic continuity and ambiguity.
For instance, as shown in Figure~\ref{fig:intro}(a), regarding the triplet (reference image: shirt, modification text: change to short sleeves, target image: T-shirt), it constitutes neither a perfectly ``clean'' sample (T-shirt $\neq$ short-sleeved shirt) nor a purely ``noisy'' sample, as they are highly correlated in attributes such as ``upper garment'' and ``short sleeves''. Such ``partial matches''\footnote{Diverging from the definition in TME~\cite{TME}, we define ``partial match'' as a scenario where only one component in multimodal query ($ref$ or $mod$) aligns with the target (simulated by shuffling $ref$ or $mod$), whereas a complete mismatch occurs when neither aligns (simulated by shuffling $tar$).} or ``hard negative samples'' violate the ``small loss hypothesis''. A model under training is likely to assign an extremely low loss value to such samples due to superficial matching (e.g., ``short sleeves''), thus erroneously classifying it as a ``clean'' sample.

Such unreliable noise identification directly leads to catastrophic representation pollution. When the model (i.e., the learner) simultaneously relies on its own unreliable outputs to act as the arbiter (estimating confidence via GMM), it falls into a fatal vicious cycle of self-dependency. As illustrated in Figure~\ref{fig:intro}(b), (1) pseudo-clean samples are erroneously trusted; (2) the model is compelled to align representations of shirts and T-shirts; and (3) this collapse of representation (resulting in decreased accuracy) conversely further deteriorates fitting of GMM, ultimately polluting entire representation space. Consequently, it is imperative to decouple the arbiter from the learner to break this cycle.

However, implementing a reliable decoupled arbiter encounters a challenging practical dilemma. Multimodal Large Language Models (MLLMs)~\cite{qwen2,cogvlm,liu2025uniform,wang2026fbs,li2025curriculum,song5,ma2026stableexplainablepersonalitytrait,fu2026maspo,zhang2026expseek,liu2024synthvlm,wang2026tracking,li2025multi,li2023ultrare}, by virtue of their robust semantic understanding capabilities, serve as ideal external ground-truth arbiter. However, their prohibitively high inference overhead precludes them from being invoked online during training. Consequently, the core challenge in addressing NTC lies in designing a paradigm that not only leverages expert-level judgment of MLLMs but also circumvents their substantial online overhead.

To address the aforementioned issues, we propose an \textbf{A}rb\textbf{I}te\textbf{R} calibrated \textbf{K}nowledge i\textbf{N}ternalizing r\textbf{O}bust net\textbf{W}ork (\textbf{Air\mbox{-}Know}).
As shown in Figure~\ref{fig:intro}(c), (a) we firstly design the \textit{External Prior Arbitration (EPA)} module, which utilizes MLLMs as an offline expert arbiter to construct a high-precision, small-scale anchor dataset $D_{anchor}$. (b) Second, we design the \textit{Expert-Knowledge Internalization (EKI)} module. This module efficiently leverages $D_{anchor}$ to guide a lightweight and robust proxy, enabling it to internalize the discriminative logic of the expert and provide reliable matching confidence. (c) Finally, during the training of the main model, we design the \textit{Dual-Stream Reconciliation (DSR)} module. This module utilizes the matching confidence of the EKI arbiter as a dynamic gating signal to divert the training into a clean alignment stream and a feedback reconciliation stream. Specifically, the clean alignment stream is used to optimize model training, while the feedback reconciliation stream is employed to correct the base representation model. The above three-phase paradigm thoroughly decouples the arbiter from the learner, enabling the model to achieve superior performance in NTC scenarios characterized by semantic ambiguity.

Our contributions are summarized as follows:
\begin{itemize}
    \item We propose Air-Know, a novel three-phase decoupled paradigm. By innovatively combining offline external prior arbitration with online expert-knowledge internalization, this paradigm efficiently achieves the decoupling of the arbiter and the learner in the NTC scenario of Composed Image Retrieval for the first time.
    
    \item We design a feedback mechanism that utilizes an independent confidence arbiter to reverse-calibrate and optimize the base representation model, enabling it to efficiently identify and rectify hard negative samples and partially matched NTC sample.

    \item Extensive experiments on multiple Composed Image Retrieval benchmark datasets, under both Noisy Triplet Correspondence and traditional settings, demonstrate that our method significantly outperforms all existing state-of-the-art methods while maintaining strong competitiveness in traditional Composed Image Retrieval tasks.
\end{itemize}

\section{Related Work}
\statement{Composed Image Retrieval with Noisy Correspondence.}
Composed Image Retrieval (CIR)~\cite{PAIR,MEDIAN} aims to retrieve a target image via a multimodal query, including a reference image and a modification text~\cite{FashionIQ,cirr,FineCIR}. Research on CIR is expected to contribute to various applications, such as semantic understanding~\cite{wang2024computing,liu2026chartverse,liu2025stole,xiao2026not,jiang2026foeforesterrorsmakes,zhang2024cf,wang2026eeo,lai2026transformers,yu2025visualizing}, and multimodal learning~\cite{lin2026mmfinereason,yuan2025video,zhang2026decoding,zhang2023multi,cheng2026enhancing,wang2024twin,10888444,zeng2025janusvln}.
Recent methods typically utilize pre-trained models such as CLIP~\cite{clip} and BLIP-2~\cite{blip-2} for feature alignment and composition~\cite{REFINE,retrack}, achieving significant progress. However, the problem of Noisy Triplet Correspondence (NTC)~\cite{TME}, which is prevalent in real-world data, remains inadequately addressed. Unlike traditional Noisy Correspondence Learning (NCL)~\cite{noisy-label-1, noisy-label-2, noisy-cross, noisy-reid, noisy-vlm}, NTC in CIR involves semantic inconsistency within triplets, presenting greater complexity than the false positive sample problem~\cite{cala,css-net,HINT} previously investigated in CIR tasks. Recently, works such as TME have begun to address this issue. 
Although TME~\cite{TME} has introduced specialized matching mechanisms to explicitly accommodate NTC in CIR, its identification of NTC still primarily draw upon the small loss hypothesis from Noisy Correspondence Learning~\cite{noisy-nips, matching-NDC-1, matching-NDC-2}, which assumes that clean samples produce low losses while noisy samples generate high losses, and utilize this premise (via GMM) to partition samples. 
While the issue of unreliable noise identification leading to error accumulation has been recognized and mitigated in other domains, such as using co-teaching~\cite{co-teaching} or confident mask mechanisms~\cite{cui2024correlation,jing2023category,jing2023multimodal}, these strategies are difficult to directly apply to the NTC scenario in CIR~\cite{TME}.
However, due to the phenomenon of semantic ambiguity, the aforementioned methods struggle to effectively resolve the NTC problem in CIR. In contrast, our proposed Air-Know framework effectively resolves the issue of representation pollution by constructing an Expert-Proxy-Diversion three-phase learning framework.

\statement{Uncertainty Modeling and Variational Inference.} 
Bayesian Neural Networks (BNNs)~\cite{bnn2} provide a rigorous framework for uncertainty modeling in deep learning~\cite{yuan2025autodrive,dong2025aurora,yuan2026if,zeng2025FSDrive,Xie_2025_ICCV} by learning the posterior distribution of parameters rather than relying on deterministic point estimation~\cite{bayesian_uncertainty1,song13,bayesian_uncertainty2,bayesian_uncertainty3,song9,yu2025iidm,ma2025mutuallearninghashingunlocking,dong2026neureasonerexplainablecontrollableunified}. However, the practical implementation of BNNs faces a significant challenge: the posterior distribution of parameters in complex deep networks is almost always computationally intractable~\cite{bnn_posterior1,bnn_posterior2}. Variational Inference (VI)~\cite{variational_inference_proposed} is a prominent class of approximate inference algorithms proposed to address this issue. By introducing a parameterized variational distribution, it transforms the inference problem into an optimizable objective, thereby becoming the most prevalent approximate solution~\cite{bayesian_uncertainty1,VI1,VI2}. Our work is closely related to this context. Due to the sparsity of the anchor dataset, and to learn a robust noisy geometric discrimination boundary, we employ the VI framework to model the posterior distribution of network parameters.
\section{Methodology}

\begin{figure*}[ht!]
  \centering
  % \fbox{\rule{0pt}{2in} \rule{0.9\linewidth}{0pt}}
     \vspace{-5pt}
   \includegraphics[width=0.97\linewidth]{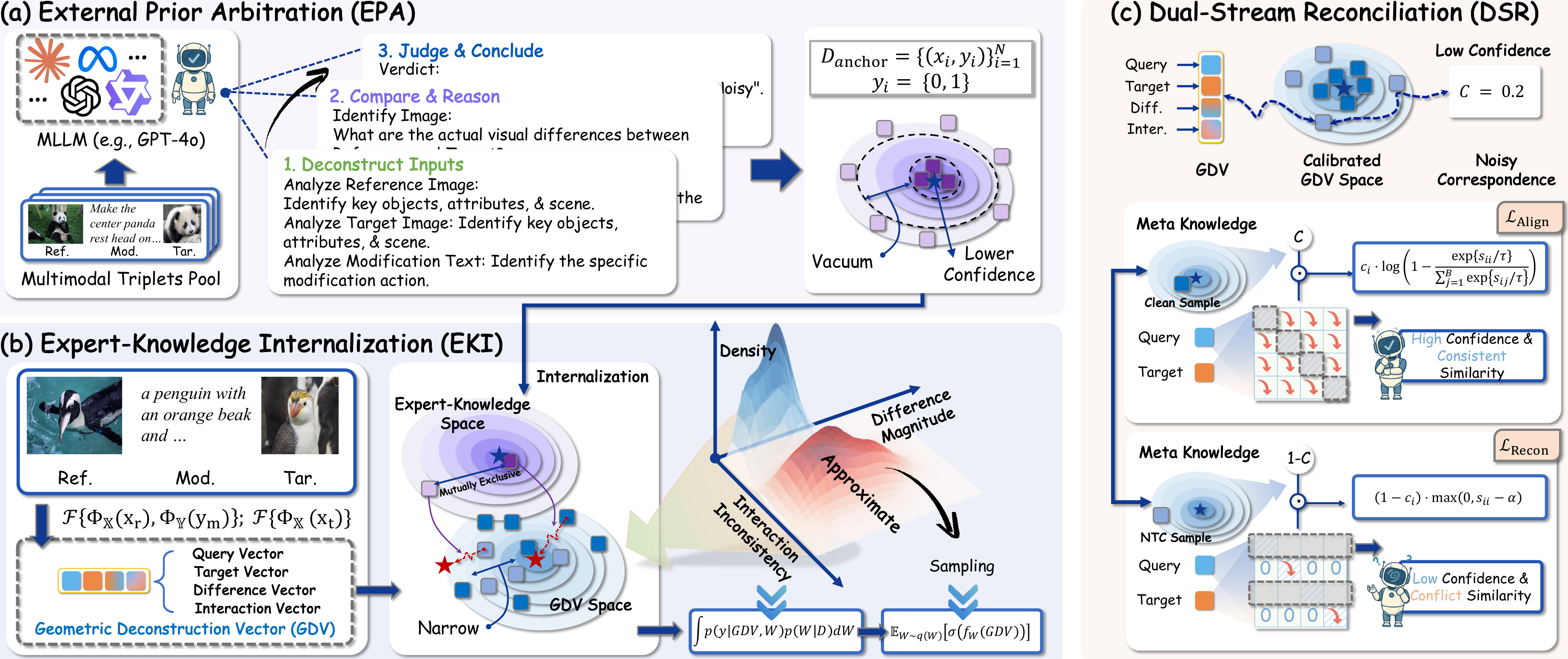}
   \caption{The proposed Air-Know consists of three primary modules: (a) External Prior Arbitration leverages an offline multimodal expert to generate reliable arbitration priors for CIR triplets, bypassing the unreliable small-loss hypothesis. (b) Expert-Knowledge Internalization transfers these priors into a lightweight proxy network, structurally preventing the memorization of ambiguous partial matches. Finally, (c) Dual-Stream Reconciliation dynamically integrates the internalized knowledge to provide robust online feedback, guiding the final representation learning. Figure best viewed in color.}
   \vspace{-15pt}
   \label{fig:framework}
\end{figure*}
We propose Air-Know, a three-phase framework. As illustrated in Figure~\ref{fig:framework}, Air-Know consists of three primary modules: (a) \textit{External Prior Arbitration} (EPA, introduced in Sec~\ref{sec:EPA}), (b) \textit{Expert-Knowledge Internalization} (EKI, introduced in Sec~\ref{sec:EKI}), (c) \textit{Dual-Stream Reconciliation} (DSR, introduced in Sec~\ref{sec:DSR}).

\subsection{Problem Formulation}
\label{subsubsec:Problem Formulation}
The Composed Image Retrieval task aims to retrieve a unique matching target image $I_t$ from a large-scale image gallery $\mathcal{D}_{gallery}$ based on a multimodal query (i.e., a reference image $I_r$ and a modification text $T_m$). A training sample is typically represented as a triplet $(I_r, T_m, I_t)$. To simulate the scenario where the training set contains extensive noisy correspondence (i.e., the Noisy Triplet Correspondence scenario), following previous studies~\cite{TME}, we simulate noisy correspondence by randomly shuffling the elements of a certain proportion of samples within the training set. Therefore, our objective is to learn a robust embedding function $\mathcal{F}$ that aligns the mapped multimodal query feature $\mathbf{z}_q$ with the corresponding target feature $\mathbf{z}_t$ in a shared metric space $\mathbb{R}^D$, while minimizing the interference of noisy correspondence during training.

\subsection{External Prior Arbitration (EPA)}
\label{sec:EPA}
To avoid the interference of noisy correspondence on training, it is critical to correctly distinguish between clean and noisy triplets. However, in the NTC problem of the CIR task, such ground truth labels are completely absent. Consequently, we employ the MLLM as an offline expert arbiter to annotate a small portion of data offline, thereby constructing a high precision and small scale anchor dataset $D_{anchor}$. To achieve the aforementioned goal, we propose a three stage cross validation strategy. Specifically, we utilize GPT-4o~\cite{gpt4} to execute the following sequential tasks.
\noindent \textbf{Step1: Deconstruct Inputs}. Initially, the MLLM independently deconstructs the three inputs of the triplet. It separately parses the visual content of the reference image $I_r$ and the target image $I_t$, while comprehending the modification intent of the given instruction $T_m$.
\noindent \textbf{Step2: Compare \& Reason}. This constitutes the core phase of the determination. By comparing ${I}_{r}$ and ${I}_{m}$, the MLLM infers the actual visual changes that have occurred, denoted as the inferred instruction $\Delta {T_I}$. Subsequently, it performs cross-validation to compare the semantic consistency between the given instruction $T_m$ and the inferred instruction $\Delta {T_I}$. This verification allows for the disregard of minor discrepancies, such as fine-grained pose adjustments. If the two remain unmatched, the MLLM further diagnoses the root cause of the NTC. For instance, it may determine the issue as an instruction mismatch (where $T_m$ is completely unrelated to $\Delta {T_I}$) or a reference image mismatch (where $I_r$ is irrelevant to the transformation process described by $T_m$ and $I_t$).
\noindent \textbf{Step3: Judge \& Conclude}. Finally, based on the reasoning and diagnosis from Step 2, the MLLM outputs a final binary determination label for the triplet, classifying it as either Clean (indicating a successful match) or Noisy (indicating a match failure with a diagnosed NTC type).

The aforementioned process is formulated as follows,
\begin{equation}
    y_i = MLLM(I_r,T_m,I_t),
\end{equation}
where $y \in \{0, 1\}$, denotes NTC sample and clean sample, respectively.

\statement{Anchor Dataset Generation.}
We randomly sampled a certain proportion of data from the training dataset, following the aforementioned process, obtained the anchor dataset,
\begin{equation}
\mathcal{D}_{anchor} = \{({I}_{r}^{(j)},{T}_{m}^{(j)},  {I}_{t}^{(j)}), y_j\}_{j=1}^{M},
\end{equation}
where $M$ denotes the number of samples in $\mathcal{D}_{anchor}$.

\subsection{Expert-Knowledge Internalization (EKI)}
\label{sec:EKI}
Given that the high inference cost of EPA limits its real-time application in large-scale training, we design the Expert-Knowledge Internalization (EKI) module as a lightweight proxy. EKI aims to learn and internalize the discriminative logic of the expert from the limited anchor dataset for noise discrimination on the entire training set.

\statement{Input Feature Engineering}. 
First, we leverage BLIP-2's Q-Former for multimodal feature extraction and composition. Given reference image $I_r$ and modification text $T_m$, the process is formulated as follows,
\begin{equation}
\fontsize{9pt}{9pt}
\hat{\mathbf{z}}_q=\operatorname{Q-Former}(\varPhi_\mathbb{X}(I_r),\varPhi_\mathbb{Y}(T_m)),
\end{equation}
where $\hat{\mathbf{z}}_q\in\mathbb{R}^{Q\times D}$ represent the composed query feature. $Q$ is number of Q-former's learnable queries, and $D$ is the feature dimension. $\varPhi_\mathbb{X}$ and $\varPhi_\mathbb{Y}$ denote BLIP-2's frozen visual encoder and tokenizer respectively. Similarly, for target image $I_t$, we obtain its feature $\hat{\mathbf{z}}_t\!=\!\operatorname{Q-Former}(\varPhi_\mathbb{X}(t))\!\in\!\mathbb{R}^{Q\times D}$. We then obatin the average pooled features $\mathbf{z}_q = AvgPool(\hat{\mathbf{z}}_q)$ and $\mathbf{z}_t = AvgPool(\hat{\mathbf{z}}_t)$.

Since $\mathbf{z}_q$ and  $\mathbf{z}_t$ contain only isolated semantics and lack explicit interaction signals to directly indicate matching degree, we construct Geometric Deconstruction Vector (GDV), $\mathbf{z}_{GDV}$.
It captures matching evidence from two complementary dimensions of global discrepancy and fine-grained commonality. GDV is formulated as follows,
\begin{equation}
    \mathbf{z}_{GDV} = \text{Concat}(\mathbf{z}_q, \mathbf{z}_t, \mathbf{z}_q - \mathbf{z}_t, \mathbf{z}_q \odot \mathbf{z}_t),
\end{equation}
where $\mathbf{z}_q, \mathbf{z}_t \in \mathbb{R}^D$ are the query feature and target feature from the backbone network, respectively.

\statement{Bayesian Geometric Separator}. 
Inspired by the geometric separation properties of ReLU networks~\cite{relu_mlp}, we employ a multi-layer ReLU-MLP as a lightweight proxy for the EKI module. Specifically, we denote the MLP function as $f_{\mathbf{W}}(\cdot)$, where $\mathbf{W}$ represents the learnable parameters. The goal of $f_{\mathbf{W}}$ is to learn a non-linear decision boundary within the $\mathbf{z}_{GDV}$ space that effectively partitions the feature manifold into two semantic regions:
\begin{itemize}
    \item \textbf{Forward Region} $\mathcal{S}^+_{\mathbf{W}}$: A region characterized by positive network activations, where samples are identified as reliable correspondence.
    \item \textbf{Suppressed Region} $\mathcal{S}^-_{\mathbf{W}}$: A region where mismatched or ambiguous features are collapsed, resulting in non-positive outputs.
\end{itemize}

Formally, for an input $\mathbf{z}$, the EKI module learns to map clean triplets (i.e., $y=1$ in $\mathcal{D}_{anchor}$) into $\mathcal{S}^+_{\mathbf{W}}$ and push noisy correspondence into $\mathcal{S}^-_{\mathbf{W}}$. This ensures that the finalized prediction $\hat{y} = \sigma(f_{\mathbf{W}}(\mathbf{z}))$ satisfies the arbitration logic, where clean samples yield $\hat{y} > 0.5$ and noisy ones yield $\hat{y} \leq 0.5$. This parameterization allows the EKI module to internalize expert knowledge into $\mathbf{W}$, providing a robust prior for the subsequent representation learning.

However, the extreme sparsity of $\mathcal{D}_{anchor}$ results in severe ill-posedness, i.e., multiple sets of parameters may successfully separate $\mathcal{D}_{anchor}$, yet their generalization capabilities beyond $\mathcal{D}_{anchor}$ can vary significantly. 
Therefore, to identify the most robust geometric boundary, we avoid point estimation for the parameter $\mathbf{W}$ and instead employ Bayesian inference to model its posterior distribution $p(\mathbf{W} | \mathcal{D}_{anchor})$, formulated as follows,
\begin{equation}p(\mathbf{W} | \mathcal{D}_{anchor}) = \frac{p(\mathcal{D}_{anchor} | \mathbf{W}) p(\mathbf{W})}{\int p(\mathcal{D}_{anchor} | \mathbf{W}') p(\mathbf{W}') d\mathbf{W}'}.
\end{equation}
Since the integral in the denominator of above formula is intractable, we introduce a variational distribution $q_\theta(\mathbf{W})$ that is easy to sample from and approximate $p(\mathbf{W} | \mathcal{D}_{anchor})$ by minimizing the KL divergence with true posterior, 
\begin{equation}
\begin{split}
    \theta^* = \operatorname*{argmin}_\theta \mathcal{D}_{\text{KL}}[q_\theta(\mathbf{W}) \parallel p(\mathbf{W} | \mathcal{D}_{anchor})].
\end{split}
\end{equation}
We demonstrate that the aforementioned objective is mathematically equivalent to maximizing the Evidence Lower Bound (ELBO). This bound comprises two components: a Reconstruction Term and a Prior Matching Term, formulated as follows,
\begin{equation} 
\begin{split}
\mathcal{L}_{\text{ELBO}}(\theta) =& \underbrace{\mathbb{E}_{\mathbf{W} \sim q_\theta}[\log p(\mathcal{D}_{anchor} | \mathbf{W})]}_{Reconstruction Term} \\ &- \underbrace{\mathcal{D}_{\text{KL}}[q_\theta(\mathbf{W}) \parallel p(\mathbf{W})]}_{Prior Matching Term}.
\end{split}
\end{equation}
In the following, we detail the approximation process for the two terms, respectively. First, regarding the Reconstruction Term, which represents the expected log-likelihood, maximizing the log-likelihood is equivalent to minimizing the negative log-likelihood in practice. We approximate this term using Monte Carlo sampling, and define a weighted BCE loss for a single sample, formulated as follows,
\begin{equation} \ell_{GDV}^{(i)} = -[\omega_i \cdot y_i \log(\hat{y}_i) + (1 - y_i) \log(1 - \hat{y}_i)], \end{equation}
where $\hat{y}_i = f_{\mathbf{W}}(\mathbf{z}_{GDV}^{(i)})$ with $\mathbf{W} \sim q_\theta(\mathbf{W})$ denotes the output of a single stochastic forward pass (i.e., one Monte Carlo sample) with Dropout enabled. $\omega_i$ represents a class-balancing weight designed to address the imbalance between clean samples and noisy correspondence.
Next, regarding the Prior Matching Term, we approximate it using a standard L2 weight decay term $\lambda_{L2}$.
Consequently, we obtain the training loss for the EKI module, formulated as,
\begin{equation}
\label{eq:loss_eki}
\mathcal{L}_{\text{EKI}} = (\frac{1}{|\mathcal{D}_{anchor}|}\sum_{i\in\mathcal{D}_{anchor}}
\ell_{GDV}^{(i)})+\lambda_{L2}||\mathbf{W}||^2,\end{equation}
where $\lambda_{L2}$ is L2 weight decay coefficient. During the inference phase, we keep dropout enabled and approximate bayesian predictive distribution by performing $T$ stochastic forward passes, i.e., $T$ samplings, formulated as follows:
\begin{equation}
p(y | \mathbf{z}_{GDV}, \mathcal{D}_{anchor}) \approx \frac{1}{T} \sum_{t=1}^{T} p(y|\mathbf{z},\mathbf{W}_t),
\end{equation}
where $\mathbf{W}_t \sim q_\theta(\mathbf{W})$.
Thus, we obtain the final confidence $\hat{c}$ (i.e., the posterior probability of $y=1$), expressed as the average of the $T$ predictions, formulated as follows:
\begin{equation}
\hat{c} = \frac{1}{T} \sum_{t=1}^{T} \sigma(f_{\mathbf{W}_t}(\mathbf{z}_{GDV})).
\end{equation}

\subsection{Dual-Stream Reconciliation (DSR)}
\label{sec:DSR}
To calibrate the base representation model concurrently with model training, we utilize the confidence $\hat{c}$ output by the EKI module as a dynamic gating signal to divide the entire training samples into a clean alignment stream and a feedback reconciliation stream. Specifically, the clean alignment stream is utilized to optimize model training, while the feedback reconciliation stream is employed to calibrate the base representation model.

\statement{Clean Alignment Stream.}
The clean alignment stream aims to learn discriminative features from high-confidence samples. Inspired by RCL~\cite{RCL}, we first employ a Robust Contrastive Loss, which focuses on utilizing negative learning to mitigate false positives, thereby indirectly achieving alignment. Then we weight it with the confidence $\hat{c}$.
This mechanism effectively dynamically suppresses the gradient contribution of low-confidence noisy correspondence to achieve clean alignment, and is formulated as follows,
\begin{equation}
\mathcal{L}_{\text{Align}}\!\!=\!\!- \frac{1}{B} \sum_{i,j\neq i}^B \hat{c}_i \cdot \log (1- \frac{\exp(s(\mathbf{z}_{q,i}, \mathbf{z}_{t,j}) / \tau)}{\sum_{j}^B \exp(s(\mathbf{z}_{q,i}, \mathbf{z}_{t,j}) / \tau)}),
\label{eq:align_loss}
\end{equation}
where $\tau$ is the temperature parameter, $B$ is the batch size, and $s(\cdot, \cdot)$ denotes the cosine similarity function.

\begin{table*}[ht]
\centering
\small
\setlength{\tabcolsep}{12pt}
\caption{Performance comparison on the FashionIQ validation set in terms of R@K(\%). The best and second-best results are highlighted in \textbf{bold} and \underline{underline}, respectively.}
\vspace{-8pt}
\resizebox{0.9\linewidth}{!}{%
\begin{tabular}{c|l|cc|cc|cc|cc|ccc}
\Xhline{1.5pt}
\multirow{2}{*}{\textbf{Noise}} & \multirow{2}{*}{\textbf{Methods}} 
& \multicolumn{2}{c|}{\textbf{Dress}} 
& \multicolumn{2}{c|}{\textbf{Shirt}} 
& \multicolumn{2}{c|}{\textbf{Toptee}} 
& \multicolumn{3}{c}{\textbf{Average}} \\
\cline{3-11}
& & R@10 & R@50 & R@10 & R@50 & R@10 & R@50 & R@10 & R@50 & AVG. \\
\hline
\hline
\multirow{6}{*}{0\%}
& SPRC~\cite{sprc}~(ICLR'24) & 49.18 & \underline{72.43} & 55.64 & 73.89 & {59.35} & 78.58 & 54.92 & 74.97 & 64.85 \\
  & QuRe~\cite{QuRe}~(ICML'25) & 46.80 & 69.81 & 53.53 & 72.87 & 57.47 & 77.77 & 52.60 & 73.48 & 63.04 \\
& TME~\cite{TME}~(CVPR'25) & {49.73} & 71.69 & {56.43} & {74.44} & 59.31 & \underline{78.94} &{55.15} & {75.02} & {65.09} \\
& {HABIT~\cite{habit}~(AAAI'26)} & {49.99} & 72.38 & \underline{56.62} & {74.68} & \textbf{59.51} & 78.53 & \underline{55.38} & \underline{75.20} & \underline{65.29}\\

& {INTENT~\cite{intent}~(AAAI'26)} & \underline{50.32} & {72.10} & {56.32} & \underline{74.93} & {59.28} & {78.45} & {55.31} & {75.16} & {65.24} \\

 & \cellcolor[HTML]{E0E9F7}\textbf{Air-Know (Ours)} & \cellcolor[HTML]{E0E9F7}\textbf{51.03} & \cellcolor[HTML]{E0E9F7}\textbf{73.08} & \cellcolor[HTML]{E0E9F7}\textbf{56.76} & \cellcolor[HTML]{E0E9F7}\textbf{75.05} & \cellcolor[HTML]{E0E9F7}\underline{59.47} & \cellcolor[HTML]{E0E9F7}\textbf{79.02} & \cellcolor[HTML]{E0E9F7}\textbf{55.75} & \cellcolor[HTML]{E0E9F7}\textbf{75.72} & \cellcolor[HTML]{E0E9F7}\textbf{65.73}\\
\hline
\multirow{5}{*}{20\%}
& SPRC~\cite{sprc}~(ICLR'24) & 39.81 & 62.22 & 48.58 & 66.29 & 50.48 & 70.58 & 46.29 & 66.36 & 56.33 \\
& TME~\cite{TME}~(CVPR'25)& {49.03} & 70.35 & \underline{55.84} & {73.16} & {57.22} & {78.23} & {54.03} & {73.91} & {63.97} \\
& HABIT~\cite{habit}~(AAAI'26) & \underline{49.63}& {71.34} & {55.67}& {73.19 }& \underline{58.14} & {78.32} & \underline{54.48} & {74.28} & \underline{64.38} \\

& {INTENT~\cite{intent}~(AAAI'26)} & {49.32}& \underline{71.43} & {55.32}& \underline{73.57}& {58.01} & \underline{78.46} & {54.22} & \underline{74.49} & {64.36} \\
& \cellcolor[HTML]{E0E9F7}\textbf{Air-Know (Ours)} & \cellcolor[HTML]{E0E9F7}\textbf{50.22}& \cellcolor[HTML]{E0E9F7}\textbf{73.48} & \cellcolor[HTML]{E0E9F7}\textbf{56.38}& \cellcolor[HTML]{E0E9F7}\textbf{74.73}& \cellcolor[HTML]{E0E9F7}\textbf{58.95} & \cellcolor[HTML]{E0E9F7}\textbf{78.94} & \cellcolor[HTML]{E0E9F7}\textbf{55.18} & \cellcolor[HTML]{E0E9F7}\textbf{75.71} &\cellcolor[HTML]{E0E9F7}\textbf{65.45} \\
\hline
\multirow{5}{*}{50\%}
& SPRC~\cite{sprc}~(ICLR'24) & 35.94 & 57.16 & 42.25 & 61.63 & 44.98 & 64.76 & 41.06 & 61.19 & 51.12 \\
& TME~\cite{TME}~(CVPR'25)& {46.26} & {68.27} & {53.09} & {71.88} & {55.07} & {76.59} & {51.47} & {72.25} & {61.86} \\ 
& HABIT~\cite{habit}~(AAAI'26) & {47.33} & {69.71} & \underline{53.72} & \underline{72.55} & {56.51} & \underline{77.00} & \underline{52.52} & {73.09} & {62.80} \\

& INTENT~\cite{intent}~(AAAI'26) & \textbf{47.99} & \textbf{71.24} & {52.78} & {72.48} & \underline{56.79} & {76.23} & \underline{52.52} & \underline{73.32} & \underline{62.92} \\
& \cellcolor[HTML]{E0E9F7}\textbf{Air-Know (Ours)} & \cellcolor[HTML]{E0E9F7}\underline{47.90} & \cellcolor[HTML]{E0E9F7}\underline{71.09} & \cellcolor[HTML]{E0E9F7}\textbf{53.95} & \cellcolor[HTML]{E0E9F7}\textbf{73.09} & \cellcolor[HTML]{E0E9F7}\textbf{57.71} & \cellcolor[HTML]{E0E9F7}\textbf{77.94} & \cellcolor[HTML]{E0E9F7}\textbf{53.19} & \cellcolor[HTML]{E0E9F7}\textbf{74.04} & \cellcolor[HTML]{E0E9F7}\textbf{63.61} \\
\hline
\multirow{5}{*}{80\%}
& SPRC~\cite{sprc}~(ICLR'24) & 28.41 & 50.77 & 36.21 & 54.37 & 35.90 & 59.06 & 33.51 & 54.03 & 43.77 \\
& TME~\cite{TME}~(CVPR'25)& {41.45} & {64.35} & {47.30} & {68.20} & {51.25} & {73.23} & {46.67} & {68.60} & {57.63} \\
& HABIT~\cite{habit}~(AAAI'26) & {42.04} & {65.20} & {50.12} & \underline{69.77} & {52.92} & {73.61} & {48.36} & {69.53} & {58.94} \\

& {INTENT~\cite{intent}~(AAAI'26)} & \underline{42.07} & \underline{65.58} & \underline{50.38} & {69.41} & \underline{53.09} & \underline{73.91} & \underline{48.51} & \underline{69.63} & \underline{59.07} \\
   
& \cellcolor[HTML]{E0E9F7}\textbf{Air-Know (Ours)} & \cellcolor[HTML]{E0E9F7}\textbf{42.52} & \cellcolor[HTML]{E0E9F7}\textbf{66.03} & \cellcolor[HTML]{E0E9F7}\textbf{51.31} & \cellcolor[HTML]{E0E9F7}\textbf{70.26} & \cellcolor[HTML]{E0E9F7}\textbf{53.86} & \cellcolor[HTML]{E0E9F7}\textbf{74.41} & \cellcolor[HTML]{E0E9F7}\textbf{49.23} & \cellcolor[HTML]{E0E9F7}\textbf{70.23} & \cellcolor[HTML]{E0E9F7}\textbf{59.73} \\
\Xhline{1.5pt}
\end{tabular}
}
\vspace{-16pt}
\label{tab:fiq_noise}
\end{table*}

\statement{Isolation-Reconciliation Stream.}
The feedback reconciliation stream leverages hard noisy samples, identified as noisy correspondence ($\hat{c}_i \to 0$) but exhibiting high query-target similarity, to optimize the base representation model, thereby compelling it to reduce the representation similarity of these noisy samples.
Let $C = \sum_{i=1}^B (1-\hat{c}_i)$, then, 
\begin{equation}
\small
\mathcal{L}_{\text{Recon}} = \frac{1}{C} {\sum_{i=1}^B (1 - \hat{c}_i)\cdot\max \left \{(s(\mathbf{z}_{q,i}, \mathbf{z}_{t,i}) - \alpha) / \tau , 0\right \}},
\label{eq:recon}
\end{equation}
where $\alpha$ denotes the tolerance margin. Finally, we obtain the final loss function of Air-Know as,
\begin{equation}
    \fontsize{9pt}{9pt}
    \mathbf{\Theta^{*}}=
    \underset{\mathbf{\Theta}}{\arg \min } \left( {\mathcal{L}}_{Align}+\lambda {\mathcal{L}}_{Recon}\right),
    \label{optimization}
\end{equation}
where $\mathbf{\Theta^{*}}$ denotes the learnable parameters of Air-Know, and $\lambda$ represents the trade-off hyperparameter. Furthermore, $\mathcal{L}_{EKI}$ in Eqn~\ref{eq:loss_eki} is utilized exclusively during the warm-up stage.

\section{Experiments}
In this section, we present a comprehensive experimental evaluation and analysis of Air-Know. Following the settings of the previous work~\cite{TME,habit,intent}, all ablation studies and parameter sensitivity analyses are conducted under a noise ratio of $0.2$.
\subsection{Experimental Settings}
\statement{Datasets.}
To comprehensively evaluate the performance of our proposed Air-Know, we utilize two datasets widely recognized in CIR, FashionIQ~\cite{FashionIQ} and CIRR~\cite{cirr}. These datasets represent typical scenarios within the fashion-domain and open-domain, respectively.

\statement{Implementation Details.}
\label{subsubsec:Implementation_details}
We adopt BLIP-2~\cite{blip-2} as Air-Know's backbone and train it on a single $32$GB NVIDIA v100 GPU for $10$ epochs, using the AdamW optimizer.
We set the number of Q-former's learned queries to $32$, and the embedding dimension D is set to $256$. The threshold $\alpha$ in Eqn~\ref{eq:recon} is set to $0.7$, while the loss weight $\lambda$ in Eqn~\ref{optimization} is set to $0.5$ for CIRR, and $0.6$ for FashionIQ. We use $0.1$ for MC dropout ratio.
The temperature coefficient $\tau$ in Eqn~\ref{eq:align_loss} is set to $0.07$. We introduce noise ratio $\sigma$ = {$0.2$, $0.5$, $0.8$} during training to simulate the NTC environment.
To ensure the DSR module receives a reliable dynamic gating signal, we adopt a Two-Stage Progressive Training Strategy. \textbf{Stage 1}: We optimize only the EKI module (learning rate $5\times10^{-4}$) using the high-precision, small-scale anchor dataset $D_{anchor}$ built by the EPA module. \textbf{Stage 2}: We freeze EKI's parameters and focus on training in DSR module. The learning rate is set to $2 \times 10^{-5}$ on CIRR and $1 \times 10^{-5}$ on FashionIQ.

\statement{Evaluation.}
For evaluation, we adopt the widely accepted Recall@K (R@K) as the primary performance metric. Following previous work's settings~\cite{encoder,HUD,OFFSET}, for CIRR, we evaluate R@{1, 5, 10, 50} and R$_{sub}$@{1, 2, 3} on its subset; for FashionIQ, we report the performance of R@{10, 50} across its three categories: Dress, Shirt, and Top\&Tee.

\begin{table*}[t]
\centering
\small
\setlength{\tabcolsep}{12pt}
\caption{Performance comparison on the CIRR test set in terms of R@K(\%) and R$_{sub}$@K(\%). The best and second-best results are highlighted in \textbf{bold} and \underline{underline}, respectively.} 
\vspace{-8pt}
\resizebox{0.9\linewidth}{!}{%
\begin{tabular}{c|l|cccc|ccc|c}
\Xhline{1.5pt}
\multirow{2}{*}{\textbf{Noise}} & \multirow{2}{*}{\textbf{Methods}} 
& \multicolumn{4}{c|}{\textbf{R@K}} 
& \multicolumn{3}{c|}{\textbf{R$_{sub}$@K}} 
& \multirow{2}{*}{\textbf{Avg(R@5, R$_{sub}$@1)}} \\
\cline{3-9}
& & K=1 & K=5 & K=10 & K=50 & K=1 & K=2 & K=3 & \\
\hline
\hline
\multirow{6}{*}{0\%}
& SPRC~\cite{sprc}~(ICLR'24) & 51.96 & 82.12 & 89.74 & 97.69 & {80.65} & {92.31} & 96.60 & {81.39} \\
& QuRe~\cite{QuRe}~(ICML'25) & {52.22} & {82.53} & {90.31} & 98.17 & 78.51 & 91.28 & 96.48 & 80.52 \\    
& TME~\cite{TME}~(CVPR'25) & \textbf{53.42} & \underline{82.99} & 90.24 & {98.15} & \textbf{81.04} & \underline{92.58} & \underline{96.94} & \textbf{82.01} \\
& HABIT~\cite{habit} (AAAI'26) & {52.71} & {82.64} & \underline{90.63} & {98.19} & \underline{80.99} & \textbf{92.77} & \textbf{97.00} & \underline{81.82} \\
& {INTENT~\cite{intent} (AAAI'26)} & \underline{53.37} & \textbf{83.16} & \textbf{90.73} & \underline{98.22} & {80.24} & {92.37} & {96.89} & {81.70} \\
\cline{2-10}
 & \cellcolor[HTML]{E0E9F7}\textbf{Air-Know (Ours)} & \cellcolor[HTML]{E0E9F7}{51.97} & \cellcolor[HTML]{E0E9F7}{82.25} & \cellcolor[HTML]{E0E9F7}{90.42} & \cellcolor[HTML]{E0E9F7}\textbf{98.25} & \cellcolor[HTML]{E0E9F7}{79.21} & \cellcolor[HTML]{E0E9F7}{91.94} & \cellcolor[HTML]{E0E9F7}{96.87} & \cellcolor[HTML]{E0E9F7}{80.73} \\
\hline
\multirow{5}{*}{20\%}
& SPRC~\cite{sprc}~(ICLR'24) & 45.90 & 75.86 & 83.52 & 93.37 & 78.10 & 91.40 & 96.05 & 76.98 \\
& TME~\cite{TME}~(CVPR'25) & 51.35 & 81.01 & 88.53 & 97.81 & \underline{78.46} & 91.25 & 96.39 & \underline{79.74} \\
& HABIT~\cite{habit}~(AAAI'26) & \textbf{51.68} & {81.02} & {89.24} & {97.81} & {78.20} & \underline{91.66} & \textbf{96.75} & {79.61} \\
& {INTENT~\cite{intent}~(AAAI'26)} & {51.25} & \underline{81.36} & \textbf{90.02} & \underline{98.05} & {77.95} & {91.40} & 96.46 & {79.66} \\
& \cellcolor[HTML]{E0E9F7}\textbf{Air-Know (Ours)} & \cellcolor[HTML]{E0E9F7}\underline{51.37} & \cellcolor[HTML]{E0E9F7}\textbf{81.73} & \cellcolor[HTML]{E0E9F7}\underline{89.78} & \cellcolor[HTML]{E0E9F7}\textbf{98.07} & \cellcolor[HTML]{E0E9F7}\textbf{79.06} & \cellcolor[HTML]{E0E9F7}\textbf{91.66} & \cellcolor[HTML]{E0E9F7}\underline{96.54} & \cellcolor[HTML]{E0E9F7}\textbf{80.40} \\
\hline
\multirow{5}{*}{50\%}
& SPRC~\cite{sprc}~(ICLR'24) & 39.93 & 66.00 & 73.59 & 86.48 & 75.81 & 89.21 & 95.37 & 70.90 \\
& TME~\cite{TME}~(CVPR'25) & 48.48 & 78.94 & 87.28 & {96.99} & {76.48} & {90.07} & {95.83} & {77.71} \\
& HABIT~\cite{habit}~(AAAI'26) & \underline{50.32} & {79.63} & {88.34} & {97.06} & {76.84} & \underline{90.60} & \textbf{96.27} & \underline{78.87} \\

& INTENT~\cite{intent}~(AAAI'26) & {49.78} & \underline{79.64} & \textbf{88.99} & \underline{97.37} & \underline{77.18} & {90.41} & {96.00} & {78.41} \\

& \cellcolor[HTML]{E0E9F7}\textbf{Air-Know (Ours)} & \cellcolor[HTML]{E0E9F7}\textbf{50.82} & \cellcolor[HTML]{E0E9F7}\textbf{80.42} & \cellcolor[HTML]{E0E9F7}\underline{88.37} & \cellcolor[HTML]{E0E9F7}\textbf{97.83} & \cellcolor[HTML]{E0E9F7}\textbf{77.43} & \cellcolor[HTML]{E0E9F7}\textbf{90.68} & \cellcolor[HTML]{E0E9F7}\underline{96.24} & \cellcolor[HTML]{E0E9F7}\textbf{78.93} \\
\hline
\multirow{5}{*}{80\%}
& SPRC~\cite{sprc}~(ICLR'24) & 29.95 & 51.25 & 58.51 & 73.86 & 70.22 & 86.05 & 93.21 & 60.74 \\
& TME~\cite{TME}~(CVPR'25) & {46.31} & {75.78} & {84.89} & {95.83} & {73.37} & {88.02} & {94.89} & {74.58} \\
& HABIT~\cite{habit}~(AAAI'26) & \underline{47.93} & {76.84} & {85.95} & {95.90} & \underline{74.87} & {89.08} & {95.21} & {75.86} \\

& {INTENT~\cite{intent}~(AAAI'26)} & {47.90} & \underline{78.13} & \underline{87.04} & \underline{96.47} & {73.81} & \underline{89.18} & \textbf{95.54} & \underline{75.97} \\
& \cellcolor[HTML]{E0E9F7}\textbf{Air-Know (Ours)} & \cellcolor[HTML]{E0E9F7}\textbf{48.84} & \cellcolor[HTML]{E0E9F7}\textbf{78.50} & \cellcolor[HTML]{E0E9F7}\textbf{87.10} & \cellcolor[HTML]{E0E9F7}\textbf{96.47} & \cellcolor[HTML]{E0E9F7}\textbf{75.30} & \cellcolor[HTML]{E0E9F7}\textbf{89.28} & \cellcolor[HTML]{E0E9F7}\underline{95.42} & \cellcolor[HTML]{E0E9F7}\textbf{76.90} \\
\Xhline{1.5pt}
\end{tabular}
}
\vspace{-10pt}
\label{tab:cirr-noise}
\end{table*}

\subsection{Performance Comparison}

To evaluate the robustness and generalization of Air-Know in NTC scenarios, we compare it with both ordinary methods~(SPRC~\cite{sprc}) and robust methods~(TME~\cite{TME}, HABIT~\cite{habit}, INTENT~\cite{intent}) on the FashionIQ and CIRR datasets under different noise ratios of $20$\%, $50$\%, and $80$\%. The results are presented in Table~\ref{tab:fiq_noise} and Table~\ref{tab:cirr-noise}. Based on our analysis, we draw the following observations:
\textbf{1) Ordinary methods, such as SPRC, exhibit a marked decline in performance when noise ratio increases.} This highlighting their high sensitivity to noisy correspondence. In contrast, robust methods, including Air-Know, maintain stronger robustness despite performance degradation, confirming the necessity of robust design in NTC scenarios.
\textbf{2) Air-Know attains superior robustness compared to other robust methods.}
First, Air-Know achieves superior performance across all noise settings.
It is worth noting that as the noise ratio increases, the performance advantage of Air-Know over baselines becomes increasingly pronounced, thereby verifying its superior robustness in NTC scenarios. We attribute this success to the unique three-phase architecture of Air-Know.
Unlike previous works relying on unreliable training dynamics, our EKI module provides credible confidence. And DSR effectively mitigates the issue of representation pollution. It is worth noting that Air-Know does not perform well in non-noise settings; this may be due to its specialized design for NTC scenarios, which prioritizes robustness at the expense of performance.

\subsection{Ablation Study}
We conducted comprehensive ablation experiments (Table~\ref{tab:ablation}) to evaluate the core components of Air-Know: EPA, EKI, and DSR.
Specifically, we configured the EPA, EKI, and DSR modules as follows. \textbf{(a) EPA Module}: D\#1 (w/o Cross-Verification) removes multi-step verification, using only a single global prompt to obtain labels; D\#2 (w/o EPA) completely removes EPA, forcing EKI to learn blindly.  
\textbf{(b) EKI Module}: 
D\#3 (w/ GDV\_1) substitutes the input GDV with the original triplet representation $(\mathbf{z}_r, \mathbf{z}_m, \mathbf{z}_t)$, where $\mathbf{z}_r, \mathbf{z}_m$ are also pooled reference and modification features extracted by Q-former; 
D\#4 (w/ GDV\_2) retains only the basic semantics $(\mathbf{z}_q, \mathbf{z}_t)$ in GDV; 
D\#5 to D\#7 remove the relation vector $(\mathbf{z}_q \odot \mathbf{z}_t)$, the difference vector $(\mathbf{z}_q - \mathbf{z}_t)$, and the basic semantics $(\mathbf{z}_q, \mathbf{z}_t)$, from the GDV, respectively; 
D\#8 (w/o Dropout) disables MC Dropout; 
D\#9 and D\#10 represent combined ablations.
\textbf{(c) DSR Module}: This evaluates the dynamic gating. 
D\#11 (w/o Align) removes the clean alignment stream, where all samples are processed by the feedback reconciliation stream; 
D\#12 (w/o Recon) removes the feedback reconciliation stream, where all samples are processed by the clean alignment stream; D\#13 (w/o DSR) removes the dual-stream structure, employing the standard InfoNCE loss.
Table~\ref{tab:ablation} presents the ablation results, and the analysis is as follows.

  \vspace{-8pt}
\begin{table}[ht]
  \centering
  \caption{Ablation study on FashionIQ and CIRR datasets. Best and sub-optimal results are highlighted in \textbf{bold} and \underline{underline}.}
  \vspace{-8pt}
  \resizebox{\linewidth}{!}{%
    \begin{tabular}{cc|cc|cc}
    \Xhline{1.5pt}
    \multicolumn{1}{c|}{\multirow{2}{*}{D\#}} & \multirow{2}{*}{Deriv.} & \multicolumn{2}{c|}{FashionIQ-Avg} & \multicolumn{2}{c}{CIRR-Avg} \\
\cline{3-6}    \multicolumn{1}{c|}{} &       & R@10  & R@50  & R@K   & Rsub@K \\
    \hline
    \hline
    \rowcolor{gray!15}
    \multicolumn{6}{c}{\textit{\textbf{(a) External Prior Arbitration (EPA)}}} \\
    \multicolumn{1}{c|}{1} & \multicolumn{1}{l|}{w/o Cross-Verification} & 53.37  & 73.51  & 79.95  & 88.04  \\
    \multicolumn{1}{c|}{2} & \multicolumn{1}{l|}{w/o EPA} & 53.99  & 73.27  & 79.41  & 88.14  \\
    \rowcolor{gray!15}
    \multicolumn{6}{c}{\textit{\textbf{(b) Expert-Knowledge Internalization (EKI)}}} \\
    \multicolumn{1}{c|}{3} & \multicolumn{1}{l|}{w/ GDV\_1} & 53.54  & 73.29  & 79.83  & 88.20  \\
    \multicolumn{1}{c|}{4} & \multicolumn{1}{l|}{w/ GDV\_2} & 54.29  & 74.37  & 79.74  & 88.15  \\
    \multicolumn{1}{c|}{5} & \multicolumn{1}{l|}{w/ GDV\_3} & 54.67  & 74.73  & 80.13  & 88.29  \\
    \multicolumn{1}{c|}{6} & \multicolumn{1}{l|}{w/ GDV\_4} & 54.47  & 74.49  & 79.62  & 88.86  \\
    \multicolumn{1}{c|}{7} & \multicolumn{1}{l|}{w/ GDV\_5} & 52.80  & 72.48  & 79.66  & 88.06  \\
    \multicolumn{1}{c|}{8} & \multicolumn{1}{l|}{w/o Dropout} & 53.71  & 73.75  & 79.48  & 88.35  \\
    \multicolumn{1}{c|}{9} & \multicolumn{1}{l|}{w/ GDV\_3 \& w/o Dropout} & 54.82  & 74.85  & 79.77  & 88.78  \\
    \multicolumn{1}{c|}{10} & \multicolumn{1}{l|}{w/ GDV\_4 \& w/o Dropout} & 54.55  & 74.86  & 79.48  & 88.45  \\
    \rowcolor{gray!15}
    \multicolumn{6}{c}{\textit{\textbf{(c) Dual-Stream Reconciliation (DSR)}}} \\
    \multicolumn{1}{c|}{11} & \multicolumn{1}{l|}{w/o Align} & 16.84  & 32.45  & 0.83  & 41.24  \\
    \multicolumn{1}{c|}{12} & \multicolumn{1}{l|}{w/o Recon} & 54.08  & 74.76  & 79.30  & 88.58  \\
    \multicolumn{1}{c|}{13} & \multicolumn{1}{l|}{w/o DSR} & 49.53  & 70.36  & 77.95  & 87.14  \\
    \hline
    \rowcolor[HTML]{E0E9F7}
    \multicolumn{2}{l|}{\textbf{Air-Know (Ours)}} & \textbf{55.18} & \textbf{75.71} & \textbf{80.24} & \textbf{89.09} \\
    \Xhline{1.5pt}
    \end{tabular}%
    }
    \vspace{-12pt}
  \label{tab:ablation}%
\end{table}%

\noindent\textbf{Analysis of EPA Module}. 
Removing EPA (D\#2) or multi-step verification (D\#1) significantly drops performance. This confirms the necessity of external arbitration and multi-step verification for complex NTC samples.

\noindent\textbf{Analysis of EKI Module}. 
1) On two datasets, the performance of D\#3, D\#4, and D\#7 all decline, indicating that EKI's effectiveness relies on combining basic semantics and geometric discrepancy features. And the performance decreases in D\#5 and D\#6 verify the effectiveness of local and global matching proxy signals, respectively. 
2) The significant performance degradation of D\#8 confirms that MC Dropout is vital for probabilistic boundary modeling, not just regularization. 
3) The results of D\#9 and D\#10 further reveal that only by combining geometric features with uncertainty estimation can EKI precisely simulate the discriminative logic of the expert.

\noindent\textbf{Analysis of DSR Module}. 
1) D\#13's performance drop illustrates the necessity of the decoupled architecture.
2) The performance of D\#11 deteriorates to $16.84$ on FashionIQ-Avg R@$10$ and $0.83$ on CIRR R@K, respectively, indicating that removing the clean alignment stream causes the model to lose the primary feature learning path.
3) D\#12's performance shows that the feedback reconciliation stream is also critical, as its removal impairs robustness by discarding valid information and preventing feedback correction, despite avoiding some representation pollution.

         \vspace{-5pt}
\begin{figure}[ht]
  \centering
  % \fbox{\rule{0pt}{2in} \rule{0.9\linewidth}{0pt}}
        \vspace{-5pt}
   \resizebox{\linewidth}{!}{\includegraphics[width=\linewidth]{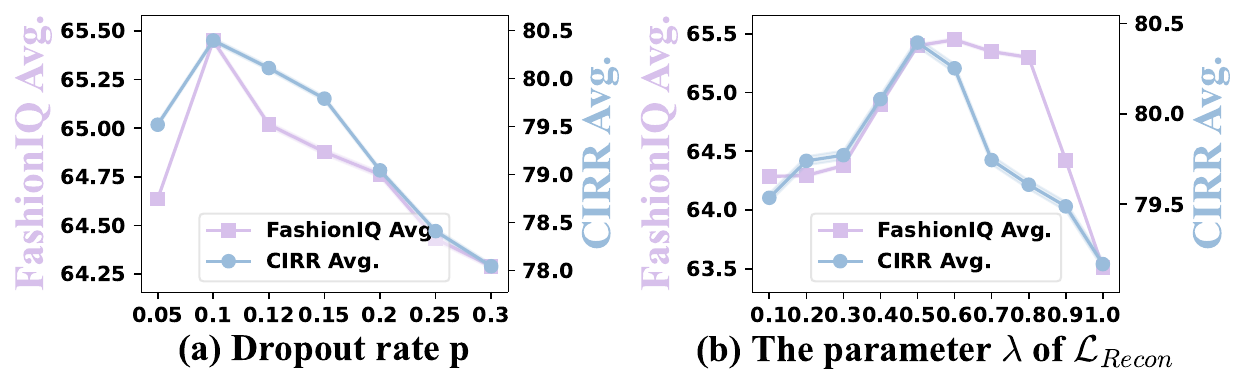}}
    \vspace{-18pt}
    \caption{Sensitivity to the hyperparameters (a) $p$ and (b) $\lambda$.}
    \vspace{-18pt}
    \label{fig:sensitivity}
\end{figure}

\subsection{Sensitivity Analysis}

To evaluate the sensitivity of Air-Know to key hyperparameters, we specifically analyzed two core parameters, the MC Dropout rate $p$ and the feedback reconciliation stream loss weight $\lambda$, on the FashionIQ and CIRR datasets.
As shown in Figure~\ref{fig:sensitivity}(a), we adjusted $p$ within the interval $[0.05, 0.3]$. 
The performance exhibits an initial increase followed by a decrease as $p$ increases, reaching a peak at $p=0.1$ in both cases. An excessively low $p$ degenerates the model to deterministic point estimation, while an excessively high $p$ injects destabilizing noise that undermines the stability of posterior sampling. 
Both extremes prevent EKI from accurately identifying the geometric truncation cone, subsequently reducing confidence
$\hat{c}$ accuracy and disrupting DSR's gradient allocation.
Figure~\ref{fig:sensitivity}(b) analyzes the impact of $\lambda \in [0.1, 1.0]$. The performance showed a similar increase-then-decrease trend, peaking at $\lambda$ = 0.6 (FashionIQ) and $\lambda$ = 0.5 (CIRR). This parameter balances the alignment and reconciliation tasks. A low $\lambda$ provides insufficient constraint from the feedback reconciliation stream, failing to neutralize noisy correspondence. Conversely, high $\lambda$ causes over-focus on reconciliation task, which subsequently impairs clean alignment stream's performance.

\subsection{Case Study}

To intuitively evaluate the effectiveness of Air-Know, Figure~\ref{fig:case} visualizes the Top-5 retrieval comparison between Air-Know and the sub-optimal baseline TME on the CIRR and FashionIQ datasets. As shown in the figure, Air-Know successfully retrieves the target to the Top-1 position in all cases. This advantage is primarily attributed to the training phase, where the EKI module, guided by the anchor dataset provided by EPA, maps samples into a geometrically separable space to construct an optimal geometric truncation cone. The resulting reliable confidence serves as a dynamic gating signal, which drives the DSR module to divert noisy correspondence to the feedback reconciliation stream for feedback optimization, thereby safeguarding the purity of the clean alignment stream. Consequently, Air-Know more accurately captures core intentions involving semantic replacement (e.g., changing a church to a bar) and attribute modification (e.g., changing a long dress to a short skirt). In contrast, TME fails to retrieve the correct target, likely because it lacks an effective decoupling mechanism for NTC discrimination, making it difficult to extract precise cross-modal correspondence. 
\begin{figure}[htbp]
    \centering
    \vspace{-8pt}
    \resizebox{\linewidth}{!}{\includegraphics[width=0.95\linewidth]{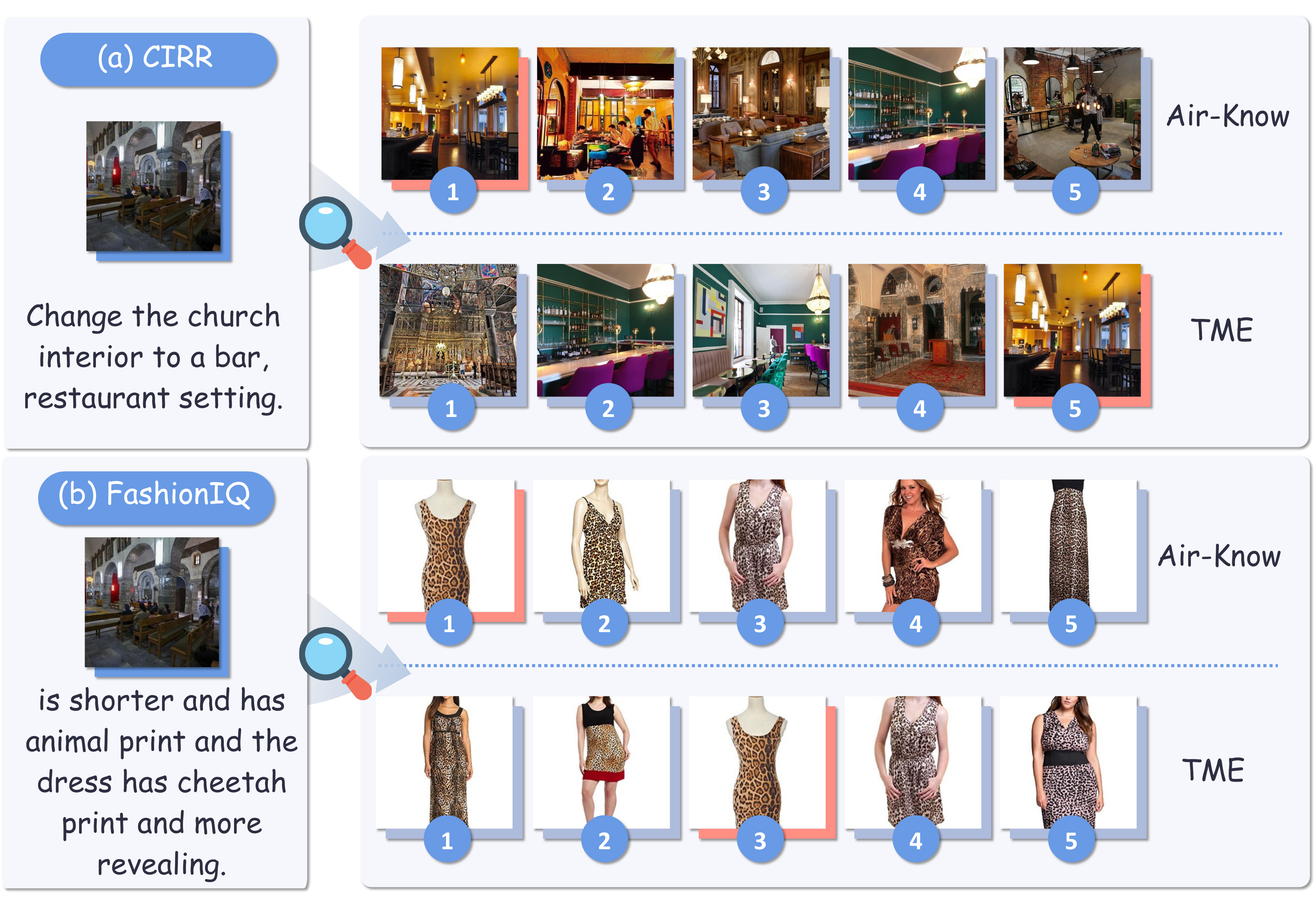}}
        \vspace{-16pt}
    \caption{Case Study on (a) CIRR and (b) FashionIQ.}
        \vspace{-18pt}
    \label{fig:case}
\end{figure}

\section{Conclusion}
In this paper, we investigate the problem of NTC in the CIR task. We reveal that the small loss hypothesis relied upon by existing robust methods faces fundamental limitations under the semantic ambiguity noise unique to NTC. This leads the model into a vicious cycle of self-dependency and triggers representation pollution. To address this problem, we propose a novel three-stage learning framework, Air-Know, which aims to leverage the expert judgment of MLLMs while avoiding their high online overhead. Simultaneously, it calibrates the base representation model online through a stream separation mechanism to prevent representation pollution. Extensive experiments on multiple CIR benchmark datasets demonstrate the superior performance, which outperforms all SOTA methods.

\clearpage
\appendix
\setcounter{page}{1}
\maketitlesupplementary
\setcounter{equation}{0}

\noindent This is the appendix of ``Air-Know: Arbiter-Calibrated Knowledge-Internalizing Robust Network for Composed Image Retrieval''. 
\begin{itemize}
    \item \textbf{Appendix~\ref{sup:proof}}: Proof
    \item \textbf{Appendix~\ref{sup:datasets}}: Datasets
    \item \textbf{Appendix~\ref{sup:cross-vali}}: Cross-Validation of EPA 
    \item \textbf{Appendix~\ref{sup:training_details}}: Training Details
    \begin{itemize}
        \item \textbf{Appendix~\ref{sup:d1}}: Architecture of the Lightweight Proxy
        \item \textbf{Appendix~\ref{sup:d2}}: Two-Stage Progressive Training Strategy

    \end{itemize}
    \item \textbf{Appendix~\ref{sup:additional_quantitative_analyses}}: Additional Quantitative Analyses
        \begin{itemize}
            \item \textbf{Appendix~\ref{sup:quantitative_efficiency}}: Efficiency Evaluation
            \item \textbf{Appendix~\ref{sup:quantitative_param}}: Additional Hyperparameter Analysis

        \end{itemize}
    \item \textbf{Appendix~\ref{sup:additional_ablation}}: Additional Ablation Study
        \begin{itemize}
            \item \textbf{Appendix~\ref{sup:additional_ablation_MLLM}}: Ablation Study of MLLMs 
            \item \textbf{Appendix~\ref{sup:additional_ablation_IMA}}: Ablation Study of EPA
        \end{itemize}
    % \item \textbf{Appendix~\ref{sup:algorithm_training}}: Algorithm of Training Procedure
    \item \textbf{Appendix~\ref{sup:more_qualitative_results}}: More Qualitative Results	
    \item \textbf{Appendix~\ref{sup:prompt}}: Prompt	
\end{itemize}

\section{Proof}
\label{sup:proof}
In the main text, we assert that optimizing the EKI module via the Evidence Lower Bound (ELBO) is mathematically equivalent to minimizing the KL divergence between the variational distribution $q_\theta(\mathbf{W})$ and the intractable true posterior $p(\mathbf{W} | \mathcal{D}_{anchor})$. While this is a standard result in variational inference, we provide the complete, step-by-step derivation here to ensure the self-containment of our theoretical framework.

\begin{theorem}[ELBO-KL Equivalence]
The log-evidence (marginal likelihood) $\log p(\mathcal{D}_{anchor})$ can be rigorously decomposed into the sum of the ELBO and the KL divergence between the variational approximation and the true posterior. Consequently, since the evidence is constant with respect to the model parameters $\theta$, maximizing the ELBO is strictly equivalent to minimizing the posterior KL divergence.
\end{theorem}

\begin{proof}
Our starting point is the definition of the Kullback-Leibler (KL) divergence from the variational distribution $q_\theta(\mathbf{W})$ to the true posterior $p(\mathbf{W} | \mathcal{D}_{anchor})$.
\begin{equation}
\begin{split}
    \mathcal{D}_{\text{KL}}&[q_\theta(\mathbf{W}) \parallel p(\mathbf{W} | \mathcal{D}_{anchor})] \\
    &= \mathbb{E}_{\mathbf{W} \sim q_\theta} \left[ \log \frac{q_\theta(\mathbf{W})}{p(\mathbf{W} | \mathcal{D}_{anchor})} \right].
\end{split}
\end{equation}
Using the property of logarithms $\log(a/b) = \log a - \log b$, we expand the term inside the expectation into two separate components: entropy and cross-entropy-like terms:
\begin{equation}
\begin{split}
    \mathcal{D}_{\text{KL}}[q_\theta \parallel p(\cdot | \mathcal{D})] &= \mathbb{E}_{q_\theta} [\log q_\theta(\mathbf{W})] \\
    &\quad - \mathbb{E}_{q_\theta} [\log p(\mathbf{W} | \mathcal{D}_{anchor})].
\end{split}
\end{equation}
Next, we invoke Bayes' theorem for the posterior term: 
\begin{equation}
    p(\mathbf{W} | \mathcal{D}_{anchor}) = \frac{p(\mathcal{D}_{anchor} | \mathbf{W}) p(\mathbf{W})}{p(\mathcal{D}_{anchor})}.
\end{equation}
Substituting this expansion into the second term of Eq.~(32), and utilizing the linearity of expectation, we obtain:
\begin{equation}
\begin{split}
    \mathbb{E}_{q_\theta} &[\log p(\mathbf{W} | \mathcal{D}_{anchor})] \\
    &= \mathbb{E}_{q_\theta} \left[ \log \frac{p(\mathcal{D}_{anchor} | \mathbf{W}) p(\mathbf{W})}{p(\mathcal{D}_{anchor})} \right] \\
    &= \mathbb{E}_{q_\theta} [\log p(\mathcal{D}_{anchor} | \mathbf{W})] + \mathbb{E}_{q_\theta} [\log p(\mathbf{W})] \\
    &\quad - \mathbb{E}_{q_\theta} [\log p(\mathcal{D}_{anchor})].
\end{split}
\end{equation}
Crucially, observe that the term $\log p(\mathcal{D}_{anchor})$ (the log-evidence) depends solely on the dataset and is independent of the weight variable $\mathbf{W}$. Therefore, it acts as a constant under the expectation $\mathbb{E}_{q_\theta}$:
\begin{equation}
    \mathbb{E}_{q_\theta} [\log p(\mathcal{D}_{anchor})] = \log p(\mathcal{D}_{anchor}).
\end{equation}
Now, we substitute the expanded form back into the original KL divergence equation:
\begin{equation}
\begin{split}
    \mathcal{D}_{\text{KL}}&[q_\theta \parallel p(\cdot | \mathcal{D})] \\
    &= \mathbb{E}_{q_\theta} [\log q_\theta(\mathbf{W})] - \Big( \mathbb{E}_{q_\theta} [\log p(\mathcal{D}_{anchor} | \mathbf{W})] \\
    &\quad + \mathbb{E}_{q_\theta} [\log p(\mathbf{W})] - \log p(\mathcal{D}_{anchor}) \Big).
\end{split}
\end{equation}
Rearranging the terms to isolate the log-evidence on one side, we reveal the fundamental decomposition:
\begin{equation}
\begin{split}
    \log p(\mathcal{D}_{anchor}) &= \mathcal{D}_{\text{KL}}[q_\theta \parallel p(\cdot | \mathcal{D})] \\
    &\quad + \underbrace{\mathbb{E}_{q_\theta} [\log p(\mathcal{D}_{anchor} | \mathbf{W})]}_{\text{Reconstruction Term}} \\
    &\quad - \underbrace{\left( \mathbb{E}_{q_\theta} [\log q_\theta(\mathbf{W})] - \mathbb{E}_{q_\theta} [\log p(\mathbf{W})] \right)}_{\text{Regularization Term}}.
\end{split}
\end{equation}
We recognize that the last grouped term corresponds exactly to the KL divergence between the variational distribution and the prior $p(\mathbf{W})$:
\begin{equation}
    \mathbb{E}_{q_\theta} \left[ \log \frac{q_\theta(\mathbf{W})}{p(\mathbf{W})} \right] = \mathcal{D}_{\text{KL}}[q_\theta(\mathbf{W}) \parallel p(\mathbf{W})].
\end{equation}
Thus, we arrive at the final identity:
\begin{equation}
    \log p(\mathcal{D}_{anchor}) = \mathcal{D}_{\text{KL}}[q_\theta \parallel p(\cdot | \mathcal{D})] + \mathcal{L}_{\text{ELBO}}(\theta),
\end{equation}
where the Evidence Lower Bound (ELBO) is defined as:
\begin{equation}
\begin{split}
    \mathcal{L}_{\text{ELBO}}(\theta) &= \mathbb{E}_{q_\theta} [\log p(\mathcal{D}_{anchor} | \mathbf{W})] \\
    &\quad - \mathcal{D}_{\text{KL}}[q_\theta(\mathbf{W}) \parallel p(\mathbf{W})].
\end{split}
\end{equation}
Since $\mathcal{D}_{\text{KL}} \ge 0$ (Gibbs' inequality), $\mathcal{L}_{\text{ELBO}}$ is indeed a lower bound on the log-evidence. More importantly, because $\log p(\mathcal{D}_{anchor})$ is fixed with respect to $\theta$, maximizing $\mathcal{L}_{\text{ELBO}}(\theta)$ is strictly equivalent to minimizing the discrepancy $\mathcal{D}_{\text{KL}}[q_\theta \parallel p(\cdot | \mathcal{D})]$.
\end{proof}

\section{Datasets}
\label{sup:datasets}
In this section, we provide a detailed introduction to two benchmark datasets widely recognized in the field of Composed Image Retrieval and adopted in our experiments, specifically FashionIQ and CIRR. The details are as follows:

\begin{itemize} \item FashionIQ~\cite{FashionIQ} serves as a standard dataset for evaluating Composed Image Retrieval within the fashion domain. It comprises 77,684 fashion images crawled from the web, totaling 30,134 annotated triplets. The content is primarily divided into three core categories: dresses, shirts, and toptees. This dataset is utilized mainly to assess retrieval performance in fashion scenarios, with a particular emphasis on evaluating the ability of the model to align visual content with descriptive modification texts. \item CIRR~\cite{cirr} utilizes real-world images derived from NLVR2, which is a natural language visual reasoning dataset. CIRR comprises 36,554 annotated triplets and 21,552 images. Unlike the domain-specific nature of FashionIQ, CIRR places greater emphasis on complex interactions among multiple objects in natural environments. This characteristic effectively mitigates the risk of model overfitting to a single domain. Furthermore, it addresses the issue of substantial false negatives caused by incomplete annotations, a problem observed in FashionIQ, by including a dedicated subset for fine-grained contrastive evaluation. Consequently, CIRR represents an ideal choice for evaluating the capability of a model to handle complex scenarios, comprehend object interactions, and fuse multimodal information. \end{itemize}

\section{Cross-Validation of EPA}
\label{sup:cross-vali}
The External Prior Arbitration (EPA) module constitutes the cornerstone of the Air-Know framework. 
As articulated in the main text, our primary challenge lies in breaking the ``self-dependent vicious cycle'' between the Learner and the Arbiter. To achieve this decoupling, we introduce the EPA module, whose sole objective is to leverage Multimodal Large Language Models (MLLMs), such as GPT-4o, as an \textit{offline expert arbiter}. Its purpose is to provide high-precision, reliable binary labels (i.e., ``Clean'' vs. ``Noisy'') for a subset of the training data. The final output of this process is a small-scale, high-precision Anchor Dataset, $D_{anchor}$, which is subsequently employed in the second phase (EKI) to supervise the lightweight proxy arbiter.

The ``cross-validation strategy'' mentioned in Sec. 3.2  relies on a multi-step, logically sophisticated prompt design. While the detailed prompts provided to the MLLM are available in Appendix~\ref{sup:prompt}, we elaborate on the three core analysis stages executed:

\noindent \textbf{Step1: Deconstruct Inputs}.
The objective of this phase is to compel the MLLM to analyze each component of the triplet independently and unbiasedly, establishing a factual foundation for subsequent comparisons and reasoning: 
\begin{itemize}[leftmargin=8pt]
    \item \textbf{Input:}A multimodal triplet ($I_r, T_m, I_t$) randomly sampled from the training pool.
    \item \textbf{Process:} 1) \textit{Analyze Reference Image ($I_r$)}. The MLLM is instructed to first analyze the reference image in isolation, identifying key objects, salient attributes, background scenes, and their relationships. 2) \textit{Analyze Target Image ($I_t$)}. Subsequently, the MLLM analyzes the target image independently in the same manner. 3) \textit{Analyze Modification Text ($T_m$)}. Finally, the MLLM parses the modification text to accurately comprehend the intended modification actions (e.g., add, remove, replace, change color).
    \item \textbf{Output:} The output of this stage is a set of structured `internal factual descriptions' ($Desc_{r},Desc_{m},Desc_{t}$). These descriptions represent the MLLM's internal working state and are fed as a holistic input into the second phase.
\end{itemize}
By enforcing sequential processing, this phase effectively prevents ``jumping to conclusions'' or ``confirmation bias'' It ensures that subsequent reasoning is grounded in independent, objective observations of each component.

\noindent\textbf{Step2: Compare \& Reason}.
This phase constitutes the core of the EPA strategy, where the ``cross-validation'' occurs. The MLLM executes complex logical reasoning to judge the internal consistency of the triplet.
\begin{itemize}[leftmargin=8pt]
    \item \textbf{Input:} The structured descriptions from Step 1.
    \item \textbf{Process:} 
    1) \textit{Infer Actual Differences ($\Delta T_{I}$):} The MLLM's first reasoning step is to temporarily ignore $T_m$. It is instructed to strictly follow the Prompt requirements to compare only $Desc_{r}$ and $Desc_{t}$, aiming to objectively describe the actual visual changes that occurred from the reference to the target image. This inferred change,  $\Delta T_{I}$, represents the ``factual change'' observed by the MLLM, corresponding to the inferred instruction $\Delta T_{I}$ mentioned in the main text. 
    2) \textit{Verify ``Instruction'' vs. ``Fact'':} The MLLM performs the critical cross-validation by comparing the semantic consistency between $T_m$ (the given instruction) and $\Delta T_{I}$ (the observed fact).
    3) \textit{Apply ``Key Principles'' (Handling Ambiguity):} 
    In NTC scenarios, substantial noise is not entirely irrelevant but manifests as ``partially match''. Our Prompt enforces the principle that human annotation allows for incomplete information. For ``Clean'' samples: As long as the core change described in the modification text actually occurred between the reference and target images, the MLLM must classify it as ``Clean'', even if ``unmentioned minor discrepancies'' (e.g., slight pose changes, background shifts) exist. This ensures the expert correctly handles clean samples.
    4) \textit{NTC Cause Diagnosis:} 
    If a severe logical disconnection or contradiction exists between $T_m$ and $\Delta T_{I}$, the MLLM classifies it as a noisy triplet correspondence. However, the core difficulty lies in partially matched samples, which are neither entirely noisy nor entirely correct. To ensure the integrity and correctness of the expert logic chain, we require the MLLM to further diagnose the root cause of the NTC classification and categorize it according to the NTC definitions: 
    \begin{itemize}
        \item Mismatched Modification Text: 
        The text $T_m$ describes an operation completely unrelated to $\Delta T_{I}$. For example, the actual visual change $\Delta T_{I}$ is ``car turns blue'', but the text $T_m$ is ``add a chair''.
        \item Mismatch Reference Imgae: 
        The modification text $T_m$ and target image $I_t$ are logically self-consistent, but the reference image is completely unrelated. For instance, $T_m$ and $I_t$ correspond to ``church spire turns gold'', but the reference image $I_r$ is an unrelated ``forest''.
        \item Mismatched Target Image: 
        The reference image $I_r$ and text $T_m$ jointly point to a clear expectation, but $I_t$ does not match.
    \end{itemize}
    
    \item \textbf{Output:} 
    The output of this stage is an exhaustive ``Reasoning Chain'', which fully documents the MLLM's logic from factual inference to cross-validation, principle application, and final diagnosis. This reasoning chain serves as the input for the final phase.
\end{itemize}

\noindent\textbf{Step3: Judge \& Conclude}.
The goal of this phase is to compile the MLLM's complex internal reasoning process into a standardized final format for dataset construction.
\begin{itemize}[leftmargin=8pt]
    \item \textbf{Input:} The Reasoning Chain output from Step 2.
    \item \textbf{Process:} 
    The MLLM receives the complete reasoning chain and populates its analysis (including descriptions from Step 1 and reasoning from Step 2) into a strict structure defined by our required output format in the prompt.
    \item \textbf{Output:} The final output of the EPA Chain-of-Thought for a single triplet contains two key components:
    1) \textit{Analysis Workflow}. Contains the MLLM's complete thought process, including the structured descriptions $Desc_{r}, Desc_{m}, Desc_{t}$ from Step 1 and the Reasoning Chain from Step 2. This ensures the interpretability and traceability of the expert knowledge.
    2) \textit{Final Judgement}. 
    Contains the MLLM's final verdict, consisting of a binary decision label accompanied by a concise summary of the core rationale.
    
\end{itemize}
Ultimately, we iterate this chain-of-thought process on the sampled small-scale subset to collect labels for individual samples, thereby constructing the anchor dataset $D_{anchor}$. This successfully consolidates the judgment knowledge of the MLLM expert, preparing for the training of the EKI module and achieving the thorough decoupling of the ``Arbiter'' and the ``Learner''.

\section{Training Details}
\label{sup:training_details}

In this section, we provide a comprehensive description of the network architecture, the optimization procedure, and the proposed Two-Stage Progressive Training Strategy to ensure the reproducibility of Air-Know.

\subsection{Architecture of the Lightweight Proxy}
\label{sup:d1}
The Expert-Knowledge Internalization (EKI) module utilizes a lightweight Multi-Layer Perceptron (MLP) to internalize the expert priors. Given the backbone feature dimension $D=256$, the input Geometric Deconstruction Vector (GDV) possesses a dimension of $4D=1024$. The EKI proxy is implemented as a three-layer MLP ($1024 \rightarrow 512\rightarrow 256\rightarrow 1$). We apply ReLU activations and Dropout ($p=0.1$) after the first two linear layers, while the final layer employs a Sigmoid function to output the confidence score $\hat{c}$. This lightweight design ensures the proxy captures the non-linear geometric boundaries without introducing significant computational overhead.

\subsection{Two-Stage Progressive Training Strategy}
\label{sup:d2}
To effectively decouple the arbiter from the learner and strictly prevent the ``self-dependent vicious cycle'', we adopt a progressive training strategy consisting of two distinct stages:

\noindent\textbf{Stage 1: Expert Internalization (Warm-up).} 
In this initial phase, our primary goal is to establish a reliable arbiter. We freeze the entire backbone network and exclusively optimize the parameters of the EKI module. 
Specifically, we construct the anchor dataset $\mathcal{D}_{anchor}$ by applying the EPA module to a randomly sampled subset comprising $40$ batches, with a batch size of $256$. 
Subsequently, utilizing this small-scale $\mathcal{D}_{anchor}$, the EKI module is trained to minimize the internalizing loss $\mathcal{L}_{EKI}$ (Eq. (9) in the main text) for $2$ epochs. 
This stage enables the proxy to learn a robust probabilistic decision boundary for identifying noisy correspondence before interacting with the main model.

\noindent\textbf{Stage 2: Dual-Stream Reconciliation.}
In the second phase, we train the representation model on the full dataset. Crucially, we \textbf{freeze} the parameters of the EKI module obtained from Stage 1. This ensures the gating signals remain stable and independent of the learner's current state. The confidence scores $\hat{c}$ generated by the frozen EKI dynamically guide the training samples into either the clean alignment stream or the feedback reconciliation stream. The backbone is then optimized using the total objective (Eq. (14) in the main text).

\section{Additional Quantitative Analysis}
\label{sup:additional_quantitative_analyses}

\begin{table*}[t!]
  \centering
  \caption{We conducted a comprehensive evaluation of computational complexity and operational efficiency, encompassing FLOPs, parameter counts, GPU memory overhead, and physical runtime. Under a unified batch size setting (Batch Size = 128), compared to the strong robust baseline TME, Air-Know achieved approximately a three-fold training acceleration with the lowest FLOPs while maintaining optimal retrieval accuracy, thereby demonstrating its optimal balance between efficiency and robustness.}
    \resizebox{\linewidth}{!}{%
    \begin{tabular}{c|c|c|c|c|c|c|c|c}
    \Xhline{1.5pt}
    Type  & Method & FLOPs(G) & Parameters(M) & GPU Memory(MiB) & Test time(s/sample) & Train Time(s/iteration) & FashionIQ-Avg & CIRR-Avg \\
    \hline
    Ordinary & SPRC  & 413.38 & 915.69 & 24478(bs=128) & 0.011 & \textbf{2.624(bs=128)} & 56.33 & 76.98 \\
    \hline
    
    \multirow{2}[2]{*}{Robust} & TME   & 405.20 & \textbf{915.68} & \textbf{12405}(bs=128) & 0.124 & 7.858(bs=128) & 63.97 & 79.74 \\
           & \cellcolor[HTML]{E0E9F7}Air-Know(Ours) & \cellcolor[HTML]{E0E9F7}\textbf{402.51} & \cellcolor[HTML]{E0E9F7}915.99 & \cellcolor[HTML]{E0E9F7}16590(bs=128) & \cellcolor[HTML]{E0E9F7}\textbf{0.010} & \cellcolor[HTML]{E0E9F7}2.805(bs=128) & \cellcolor[HTML]{E0E9F7}\textbf{65.45} & \cellcolor[HTML]{E0E9F7}\textbf{80.40} \\
    \Xhline{1.5pt}
    \end{tabular}%
    }
  \label{tab:efficiency}%
\end{table*}%

\subsection{Efficiency Evaluation}
\label{sup:quantitative_efficiency}
To evaluate system efficiency and resource consumption, we compared the ordinary method SPRC, the robust method TME, and our Air-Know under identical hardware conditions with a batch size of 128. The experimental results are presented in Table~\ref{tab:efficiency}, and the specific analysis is as follows:

1) In terms of computational cost (FLOPs), Air-Know (402.51G) exhibits the lowest computational overhead, representing a reduction of approximately 0.66\% compared to TME (405.20G) and 2.63\% relative to SPRC (413.38G). This result indicates that Air-Know effectively reduces the computational load while maintaining performance. Furthermore, the parameter counts for all models remain at approximately 915M, further demonstrating that the performance improvements of Air-Know do not stem from an increase in model scale.

2) Regarding GPU memory consumption, the requirement of Air-Know (16590 MiB) is higher than that of TME (12405 MiB). This increase in memory usage is within expectations and is primarily attributed to the training mechanism of Air-Know. During the training phase, the model is required to run the backbone network in parallel with the lightweight Expert-Knowledge Internalization proxy (utilized for generating matching confidence $\hat{c}$ in real time) while simultaneously maintaining gradient computations for both the clean alignment stream and the feedback reconciliation stream. We posit that exchanging moderate memory usage for enhanced training robustness and superior final feature quality represents a reasonable trade-off under current hardware conditions.

3) In terms of training efficiency, we adopt a rigorous evaluation metric that incorporates both the one-off offline MLLM labeling process and the EKI warmup phase. Under this comprehensive setting, Air-Know records a training time of 2.805 s/iter. While this presents a slight increase over SPRC (2.624 s/iter) due to the inclusion of the aforementioned components, Air-Know still demonstrates a substantial advantage over the robust baseline TME (7.858 s/iter), achieving a $\sim$3$\times$ speedup (approx. 64.9\% reduction in training time). Regarding inference efficiency, the auxiliary EKI module is explicitly discarded during the test phase, allowing the system to rely solely on the optimized backbone. Consequently, Air-Know achieves a test time of 0.010 s/sample, surpassing TME (0.124 s/sample) by a margin of 12.4$\times$, ensuring high efficiency for real-world deployment.

\subsection{Additional Hyperparameter Analysis}
\label{sup:quantitative_param}

\begin{figure}[t!]
    \centering
    \resizebox{\linewidth}{!}{\includegraphics[width=\linewidth]{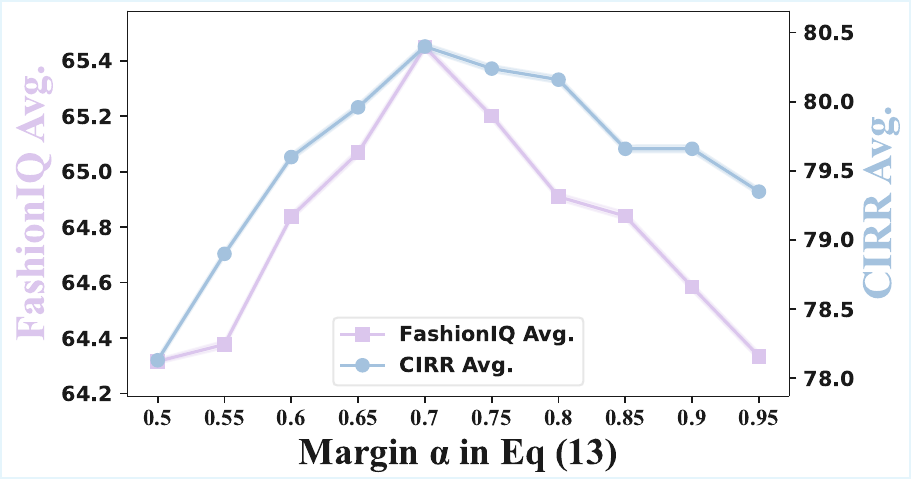}}
    \caption{Sensitivity analysis of the margin $\alpha$. We evaluated the impact of it in Equation (13), which serves as a parameter to control the threshold for penalizing noisy correspondence in the feedback reconciliation stream. A lower $\alpha$ imposes stricter filtering on semantically similar samples, while a higher $\alpha$ allows more samples exhibiting uncertainty to pass through, thereby creating distinct trade-offs between noise handling and sample retention.}
    \label{fig:alpha_sensi}
\end{figure}

We investigated the impact of the tolerance margin $\alpha$ within the feedback reconciliation stream (Equation 13 in the main text). This hyperparameter acts as a threshold to penalize high-similarity noisy correspondence, specifically samples exhibiting low matching confidence $\hat{c}$ but high similarity $s$. As illustrated in Figure~\ref{fig:alpha_sensi}, we present its sensitivity curves on the FashionIQ and CIRR datasets.

It can be observed from the figure that as $\alpha$ increases from 0.5 to 0.95, the performance curves on both datasets exhibit a distinct inverted V-shaped trend, reaching their peak performance simultaneously at $\alpha = 0.7$. We explain this phenomenon by analyzing the gradient suppression mechanism within the Noisy Triplet Correspondence (NTC) scenario as follows:

\noindent \textbf{Over-correction caused by excessively small $\alpha$:} When $\alpha < 0.7$, the tolerance threshold is set excessively low. Consequently, even weakly correlated noisy correspondence sharing only partial basic semantics, such as similar backgrounds or objects of the same category, triggers the penalty mechanism of $\mathcal{L}_{Recon}$. This overly aggressive gradient suppression strategy not only pushes away noise but also disrupts reasonable semantic continuity within the feature space. This hinders the model from learning intra-class commonalities, thereby leading to performance degradation, for instance, FashionIQ drops to approximately 64.3 when $\alpha=0.5$.

\noindent \textbf{Ineffective Constraints caused by excessively large $\alpha$:} When $\alpha > 0.7$, the tolerance threshold is set excessively high. In this case, only extreme hard noise that is highly similar to the query features triggers the loss function, while the vast majority of interfering samples with moderate similarity are ignored as they satisfy the condition $s/\tau < \alpha$ (i.e., the loss becomes 0). This results in sparse or vanishing gradients within the feedback reconciliation stream, rendering it unable to effectively execute the de-pollution task and causing gradual degradation in model performance.

Consequently, $\alpha = 0.7$ represents an optimal balance point. It allows noisy correspondence to retain reasonable basic visual similarity while effectively penalizing erroneous high-similarity matches, thereby enhancing the robustness of Air-Know against noise.

\section{Additional Ablation Study}
\label{sup:additional_ablation}

\subsection{Ablation Study of MLLMs}
\label{sup:additional_ablation_MLLM}

\begin{table*}[htbp]
  \centering
  \caption{Performance and Efficiency Analysis of MLLMs in the EPA Module. We evaluated the discrimination accuracy and inference speed of various MLLMs across different noise levels. Under a unified experimental setting with a batch size of 256, we employed 256 concurrent threads to invoke the corresponding APIs, thereby achieving an efficient experimental workflow.}
  \resizebox{\linewidth}{!}{%
    \begin{tabular}{c|l|c|c|c|c|c|c|c}
    \Xhline{1.5pt}
    Noise & MLLM   & $s/sample \downarrow$ & FIQ-$Accuracy(\%)\uparrow$ & CIRR-$Accuracy(\%)\uparrow$ & $s/batch\downarrow$ & FIQ-R@10-Avg & FIQ-R@50-Avg & CIRR-Avg \\
    \hline
    \hline
    \multirow{8}{*}{20\%}
    & gpt-5 & 20  &89.85  & \textbf{90.62} & 129 &\textbf{55.45} &\textbf{75.88} &\textbf{80.61}\\
    & gpt-5-mini  & \textbf{9}  &79.14   & 80.47 & 41 &54.96 &75.55 &80.22\\
    & Llama-3.2-Vision-90B  & 17  & 75.32  & 76.95 & 80 &54.78 &75.42 &80.10\\
    & claude-sonnet-4.5  & 17  & 85.22  & 86.00  & 57 &55.22 &75.74 &80.45\\
    & Qwen3-VL-235B-Instruct  & 32 &73.95  & 75.78 & 204 &54.73 &75.39 &80.05\\
    & gpt-5-nano & 10  &81.66  & 82.81 & 37 &55.07 &75.63 &80.31\\
    & gemini-2.5-pro & 15  &72.58  & 74.22 & 68 &54.65 &75.32 &79.98\\
    
    & \cellcolor[HTML]{E0E9F7}gpt-4o   & \cellcolor[HTML]{E0E9F7}10    & \cellcolor[HTML]{E0E9F7}84.09& \cellcolor[HTML]{E0E9F7}85.16 & \cellcolor[HTML]{E0E9F7}\textbf{20} & \cellcolor[HTML]{E0E9F7}55.18& \cellcolor[HTML]{E0E9F7}75.71& \cellcolor[HTML]{E0E9F7}80.40\\
    \hline
    \multirow{8}{*}{50\%}
    & gpt-5 & 20 & 92.51   & \textbf{93.15} & 129 &\textbf{53.68}		 &\textbf{74.22} &\textbf{79.15}\\
    & gpt-5-mini  & \textbf{9}   & 82.68  & 84.03 & 41 &53.20	 &73.88	 &78.75\\
    & Llama-3.2-Vision-90B  & 17  &78.45  & 80.12 & 80 &52.98	 &73.72	 &78.56\\
    & claude-sonnet-4.5  & 17 &89.12   & 89.95  & 57 &53.51	 &74.10	 &79.01\\
    & Qwen3-VL-235B-Instruct  & 32 &78.56  & 80.44 & 204 &53.01	 &73.75	 &78.59\\
    & gpt-5-nano & 10 &84.39   & 85.67 & 37 &53.28	 &73.94	 &78.82\\
    & gemini-2.5-pro & 15  & 75.35  & 77.08 & 68 &52.82	 &73.61	 &78.43\\
    
    & \cellcolor[HTML]{E0E9F7}gpt-4o   & \cellcolor[HTML]{E0E9F7}10 & \cellcolor[HTML]{E0E9F7}87.15   & \cellcolor[HTML]{E0E9F7}88.28 & \cellcolor[HTML]{E0E9F7}\textbf{20} &\cellcolor[HTML]{E0E9F7}53.42 &\cellcolor[HTML]{E0E9F7}74.04 &\cellcolor[HTML]{E0E9F7}78.93\\
    \hline
    \multirow{8}{*}{80\%}
    & gpt-5 & 20  & 93.88  & \textbf{94.43} & 129  &\textbf{50.18}	 &\textbf{70.72}	 &\textbf{77.05}\\
    & gpt-5-mini  & \textbf{9}   & 85.52  & 86.88 & 41 &49.75	 &70.41	 &76.68\\
    & Llama-3.2-Vision-90B  & 17 & 81.50   & 83.25 & 80 &49.54	 &70.26	 &76.50\\
    & claude-sonnet-4.5  & 17 & 91.85   & 92.64  & 57 &50.08	 &70.65	 &76.96\\
    & Qwen3-VL-235B-Instruct  & 32 & 80.35  & 82.30 & 204 &49.49	 &70.22	 &76.45\\
    & gpt-5-nano & 10  & 87.18  & 88.52 & 37 &49.85	 &70.48	 &76.78\\
    & gemini-2.5-pro & 15 & 79.42   & 81.15 & 68 &49.42	 &70.18	 &76.39\\
    
   & \cellcolor[HTML]{E0E9F7}gpt-4o   & \cellcolor[HTML]{E0E9F7}10    & \cellcolor[HTML]{E0E9F7}90.25& \cellcolor[HTML]{E0E9F7}91.41 & \cellcolor[HTML]{E0E9F7}\textbf{20} & \cellcolor[HTML]{E0E9F7}50.01 & \cellcolor[HTML]{E0E9F7}70.60 & \cellcolor[HTML]{E0E9F7}76.90\\
    \Xhline{1.5pt}
    \end{tabular}%
    }
  \label{tab:ablation_MLLM}%

\end{table*}%

In the Air-Know framework, the External Prior Arbitration (EPA) module plays a pivotal role. Its core task is to utilize the MLLM as an offline expert to construct a high-precision Anchor Dataset ($D_{anchor}$) and through this dataset guide the subsequent lightweight agent (the EKI module) in learning noise discrimination logic. 
Although the EPA operates offline, constructing $D_{anchor}$ still necessitates the processing of large volumes of data. The substantial inference overhead of MLLMs serves as the primary bottleneck. Consequently, it is necessary to find the optimal balance between expert-level discrimination quality and acceptable temporal and economic costs.

Based on the aforementioned requirements, as shown in Table~\ref{tab:ablation_MLLM}, we compared mainstream open-source MLLMs (Llama-3.2-Vision-90B, Qwen3-VL-235B-Instruct) and closed-source MLLMs (GPT-4o, Claude-Sonnet-4.5, Gemini-2.5-Pro, GPT-5, etc.).We obtained the following observations: 
\textbf{1) Complex instruction-following capability:} When executing the deconstruction-reasoning-determination chain of EPA, GPT-5 and Claude-Sonnet-4.5 demonstrated extremely high accuracy, capable of precisely identifying highly deceptive noisy correspondences. 
\textbf{2) Reasoning efficiency and throughput:} Efficiency is a key consideration in our model selection. Although GPT-5 demonstrates superior performance, its batch processing time ($s/batch$) reaches 129 seconds, which is too slow and expensive for constructing $D_{anchor}$. In contrast, GPT-4o requires only 20 seconds for batch processing, making its efficiency more than six times that of GPT-5. \textbf{3) Robustness trade-off:} Under tests with different noise ratios (20\%, 50\%, 80\%), GPT-4o exhibited excellent stability. Particularly under the extreme noise level of 80\%, GPT-4o maintained an accuracy of 91.41\%, which is only slightly lower than that of GPT-5 (94.43\%) but higher than other lightweight models. This indicates that GPT-4o can provide sufficiently clean supervision signals for the EKI module to internalize while maintaining high efficiency.

\subsection{Ablation Study of EPA}
\label{sup:additional_ablation_IMA}
\begin{table}[htbp]
  \centering
  \caption{Ablation results of EPA module on the CIRR dataset under different noise ratios ($\sigma = 0.2, 0.5, 0.8$). Specifically, w/o Step 1 removes the Deconstruct Inputs, compelling the model to execute comparisons directly without establishing objective factual anchor points. w/o Step 2 removes the Compare \& Reason, omitting the crucial logical cross-verification stage and jumping directly from observation to determination. w/o Step 1 \& 2 removes all structured intermediate steps, retaining only the naive end-to-end binary decision.}
    \begin{tabular}{l|c|c|c}
\Xhline{1.5pt} 
   Variant & $\sigma = 0.2$   & $\sigma = 0.5$   & $\sigma = 0.8$ \\
    \hline
    \hline
    w/o Step1 & 83.38 & 86.84 & 89.25 \\
    \hline
    w/o Step2 & 76.56 & 87.50 & 89.72 \\
    \hline
    w/o Step1\&2 & 75.78 & 86.72 & 87.28 \\
    \hline
    \rowcolor[HTML]{E0E9F7}
    Ours  & 85.16 & 88.28 & 91.41 \\
\Xhline{1.5pt}
    \end{tabular}
  \label{tab:MLLM_analog}
\end{table}

In the External Prior Arbitration (EPA) module, we employ a three stage cross validation strategy as follows:

\textit{Step 1:} Deconstruct Inputs. The model independently analyzes the visual content of the reference image ($I_r$) and the target image ($I_t$) and separately comprehends the modification intent of the modification text ($T_m$), establishing objective factual descriptions.

\textit{Step 2:} Compare \& Reason. The model infers the actual visual change ($\Delta T_I$) between the images and cross-verifies its semantic consistency with the textual instruction ($T_m$) to diagnose the presence of NTC noisy correspondence.

\textit{Step 3:} Judge \& Conclude. Based on the reasoning diagnosis described above, the model outputs a final binary determination (Clean or Noisy), thereby constructing a high-precision anchor dataset.

\subsubsection{Quantitative Results}

To verify the necessity of our three-stage cross-validation strategy, we conducted ablation experiments on the CIRR dataset across varying noise ratios ($\sigma \in \{0.2, 0.5, 0.8\}$). 
Our core hypothesis is that a reliable arbiter must simultaneously possess two capabilities: \textbf{1) Robustness against high-level semantic contradictions (NTC)} and \textbf{2) Tolerance for minor visual discrepancies (partial matches)}.

As shown in Table~\ref{tab:MLLM_analog}, we compared our complete prompt (Figure~\ref{fig:prompt1}) against three variants: removing independent deconstruction (w/o Step 1, Figure~\ref{fig:prompt2}), removing explicit reasoning (w/o Step 2, Figure~\ref{fig:prompt3}), and an end-to-end baseline (w/o Step 1\&2, Figure~\ref{fig:prompt4}). 
These comparisons yielded the following observations:

\noindent \textbf{1) Absence of objective anchors leads to text-induced confirmation bias:} Removing Step 1 forces the MLLM to compare inputs without establishing independent objective factual descriptions. In high-noise environments ($\sigma=0.8$), its accuracy exhibited a significant decline compared to the complete model (91.41\%), dropping to 89.25\% ($\Delta = -2.16\%$).
In NTC scenarios, particularly where the actual visual change is completely unrelated to the text instruction, misleading text ($T_m$) often causes the model to ``hallucinate'' non-existent visual changes. By generating independent factual descriptions ($Desc_r, Desc_m, Desc_t$), Step 1 effectively severs this interference, ensuring that the determination is based on objective facts.

\noindent \textbf{2) Absence of reasoning chain causes rigid rejections of valid partial matches:} Skipping the ``Compare \& Reason'' stage (w/o Step 2) caused accuracy at $\sigma=0.2$ to plummet to 76.56\%, compared to Air-Know (85.16\%). Under low-noise ($\sigma=0.2$) setting, the dataset contains a large number of partial match samples.
Lacking the reasoning process in Step 2 that infers the actual difference ($\Delta T_I$) and performs cross validation with $T_m$, the model struggles to distinguish between ``tolerable visual discrepancies'' and ``fundamental semantic contradictions'', mistakenly discarding valid samples as noise.

\noindent \textbf{3) Dual limitations of unstructured determination:} The end-to-end variant (w/o Step 1\&2) performed worst across all settings. Unstructured prompting fails the core dilemma under NTC scenarios: balancing skepticism against false text with tolerance for partial matches. Lacking Step 1, the model cannot resist text inducement in high-noise scenarios; it is extremely prone to the misdetermination of noise where ``the text describes a change but the visual change does not occur'' as Clean. Lacking Step 2, it cannot understand the rationality of ``matching primary intent but flawed details'' in low-noise scenarios; it tends to capture all minute visual discrepancies and determine them as Noisy.

\subsubsection{Qualitative Results}
To demonstrate the effectiveness of our prompt strategy, Figure~\ref{fig:prompt_ablation} visualizes a highly deceptive partially mismatched sample. The reference image ($I_r$) shows a green cantaloupe, but the text ($T_m$) reads ``Slice open the orange'', and the target ($I_t$) depicts a sliced blood orange. Despite perfect text-target alignment, the operational premise contradicts $I_r$. Thus, it is a noisy triplet. We analyze the four variants' responses below:

\begin{figure*}[t!]
    \centering
\includegraphics[width=0.95\linewidth]{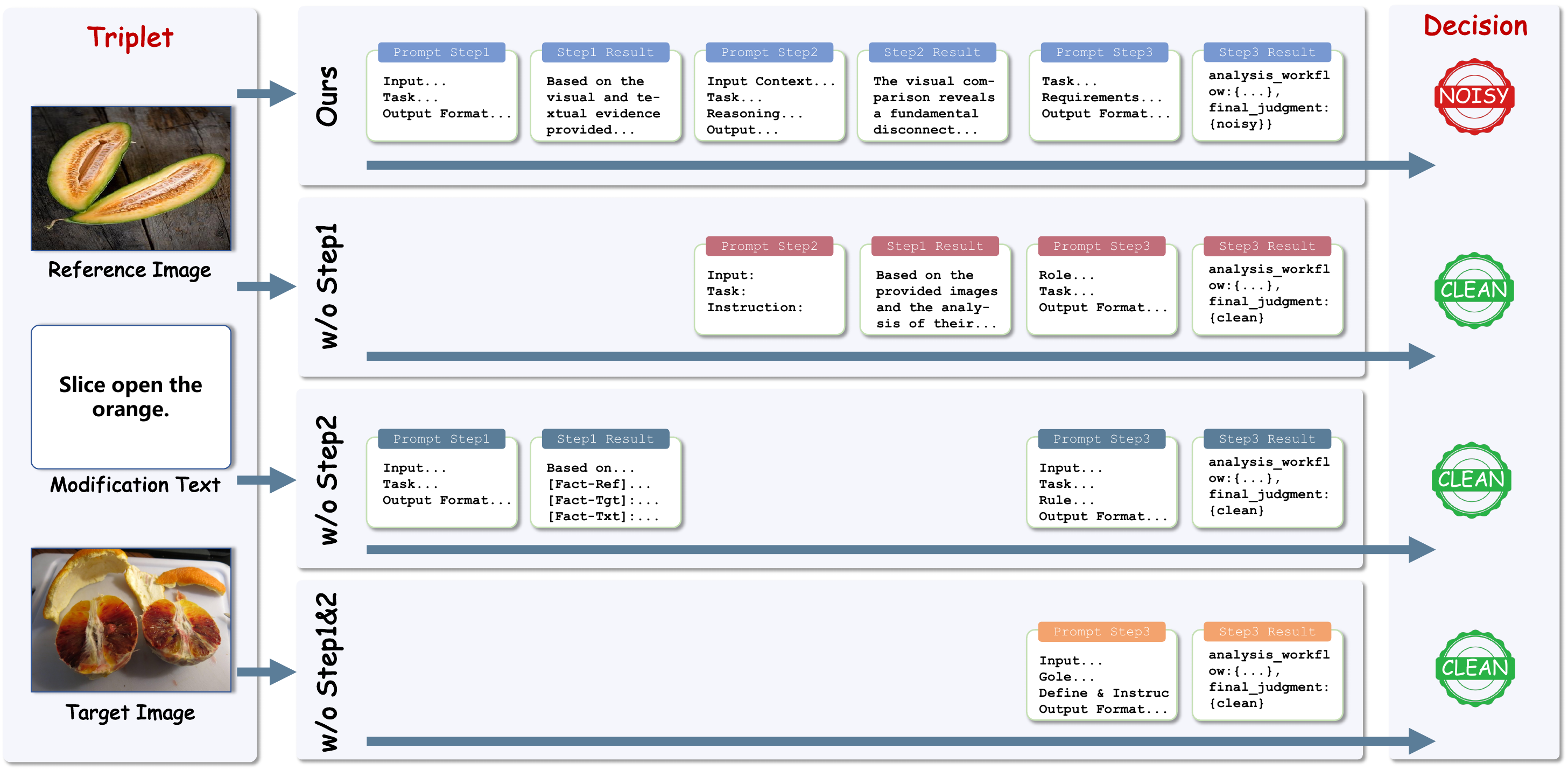}
    \caption{Visualization of prompt design and ablation study on a real-world NTC case. We present a comparison of the reasoning process between the full Air-Know (Ours) and three ablation variants on a typical ``reference image mismatch'' NTC sample. The left side displays the input triplet; the center illustrates the prompt execution flow (with specific steps retained or removed); and the right presents the final determination. In this case, only the complete model successfully identified the logical conflict between the reference image ``a sliced cantaloupe'' and the instruction ``Slice open the orange''. through the complete ``Deconstruct-Reason-Determine'' chain, while all ablation variants fell into confirmation bias and provided an incorrect ``Clean'' determination.}
    \label{fig:prompt_ablation}
\end{figure*}

\noindent \textbf {1) Three stage cross validation prompt (First row, correct determination):} 
The complete EPA model correctly labels $I_r$ as a melon in Step 1. In Step 2, it detects the conflict between the text's operational object (``orange'') and the $I_r$ anchor (melon), successfully diagnosing a Reference Image Mismatch and outputting Noisy.

\noindent \textbf{2) w/o Step 1 (second row, misdetermination):} Upon removing Step 1, the model loses its objective cognition of the reference image and succumbs to typical Confirmation Bias. Due to the absence of a pre-established factual anchor stating ``this is a cantaloupe'', the model focuses its attention entirely on the correspondence between the modification text ($T_m$) and the target image ($I_t$). It observes that the instruction ``Slice open the orange'' is perfectly executed in the target image, thereby generating a visual hallucination; consequently, it subjectively ignores the actual content of the reference image or erroneously assumes that the reference image is an unsliced orange.

\noindent \textbf{3) w/o Step 2 (Third row, misdetermination):} The failure of this variant reveals the importance of logical cross-validation in Step 2. Although the model might have identified the melon and the orange during the observation phase, the absence of an explicit reasoning mechanism for verifying the consistency between the instruction and the facts prevented the model from performing the critical logical check, specifically asking whether the instruction could be applied to the reference image. Faced with the strong semantic correlation between the text and the target image, the model, lacking logical scrutiny, tended to take shortcuts. It made determinations directly based on the matching degree between $T_m$ and $I_t$, thereby ignoring the premise conflict between $I_r$ and $T_m$. 

\noindent \textbf{4) w/o Step 1\&2 (fourth row, misdetermination):} As an end-to-end baseline, in the absence of structured guidance, the model is completely dominated by the evident semantic alignment between the text and the target. It is incapable of processing the complex dependencies within the triplet and fails to recognize the disconnection of the reference image within the logical chain, ultimately blindly rendering an erroneous determination of Clean.
This result compellingly demonstrates that relying solely on perceptual capabilities is insufficient when confronting covert noise such as Partially Mismatch. By anchoring reference facts and conducting logical verification, the EPA strategy effectively prevents the model from succumbing to misdetermination caused by local semantic matching.

\section{More Qualitative Results}

\label{sup:more_qualitative_results}

\begin{figure*}[t]
    \label{fig:propmt_a}
    \centering
    \vspace{-12pt}
    \resizebox{\linewidth}{!}{\includegraphics[width=0.95\linewidth]{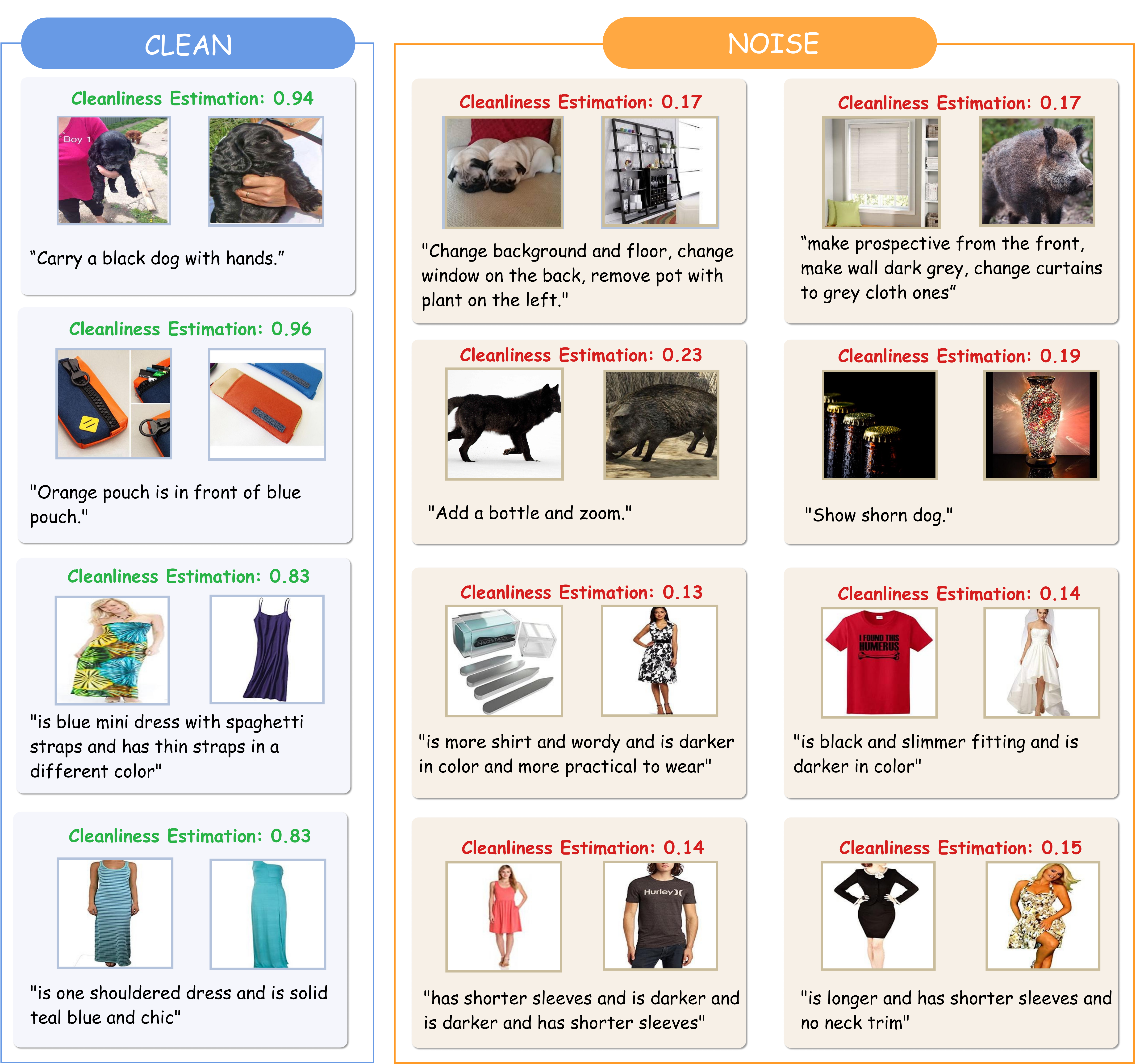}}
    \caption{Visualization of NTC recognition results by the EKI module. We present the discrimination results of the EKI module for triplets. Left (Blue): Semantically consistent samples are mapped to high-confidence regions. Right (Orange): NTC exhibiting mismatch or text-image inconsistency is precisely suppressed, receiving extremely low scores. These reliable estimation values serve as dynamic gating signals for the DSR module.}
        \vspace{-14pt}
    \label{fig:ntc}
\end{figure*}
\vspace{-6pt}

\begin{figure*}[t]
    \centering
    \resizebox{0.9\linewidth}{!}{\includegraphics[width=\linewidth]{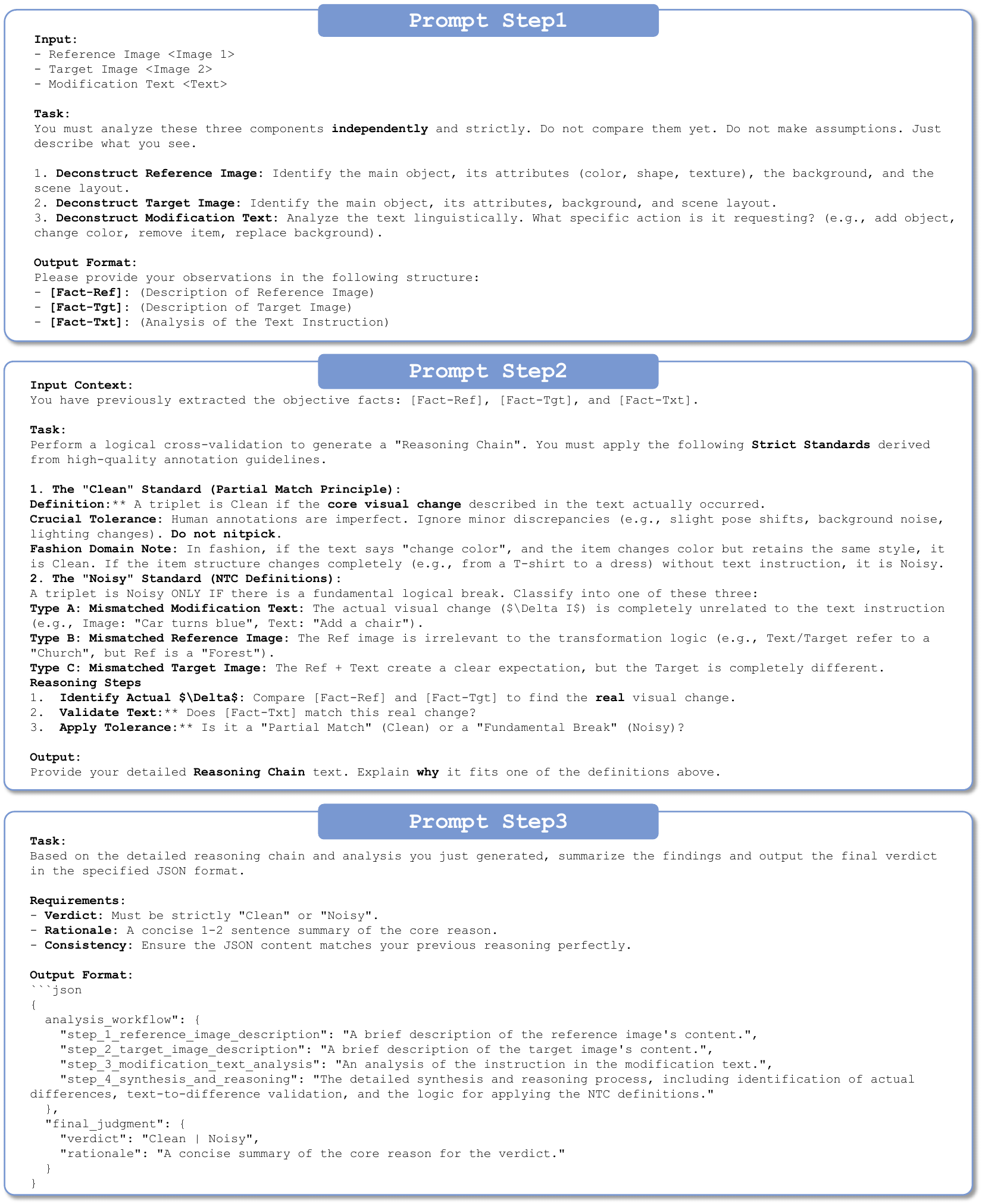}}
    \caption{The complete three stage cross-validation prompt architecture. This design enforces a Deconstruct-Reason-Determine process. Specifically, it guides the model to first deconstruct Inputs to establish objective visual factual anchors, subsequently to diagnose NTC types and verify semantic consistency, and finally to output a determination based on a comprehensive chain of evidence.}
    \label{fig:prompt1}
\end{figure*}

\begin{figure*}[t]
    \centering
    \resizebox{0.9\linewidth}{!}{\includegraphics[width=\linewidth]{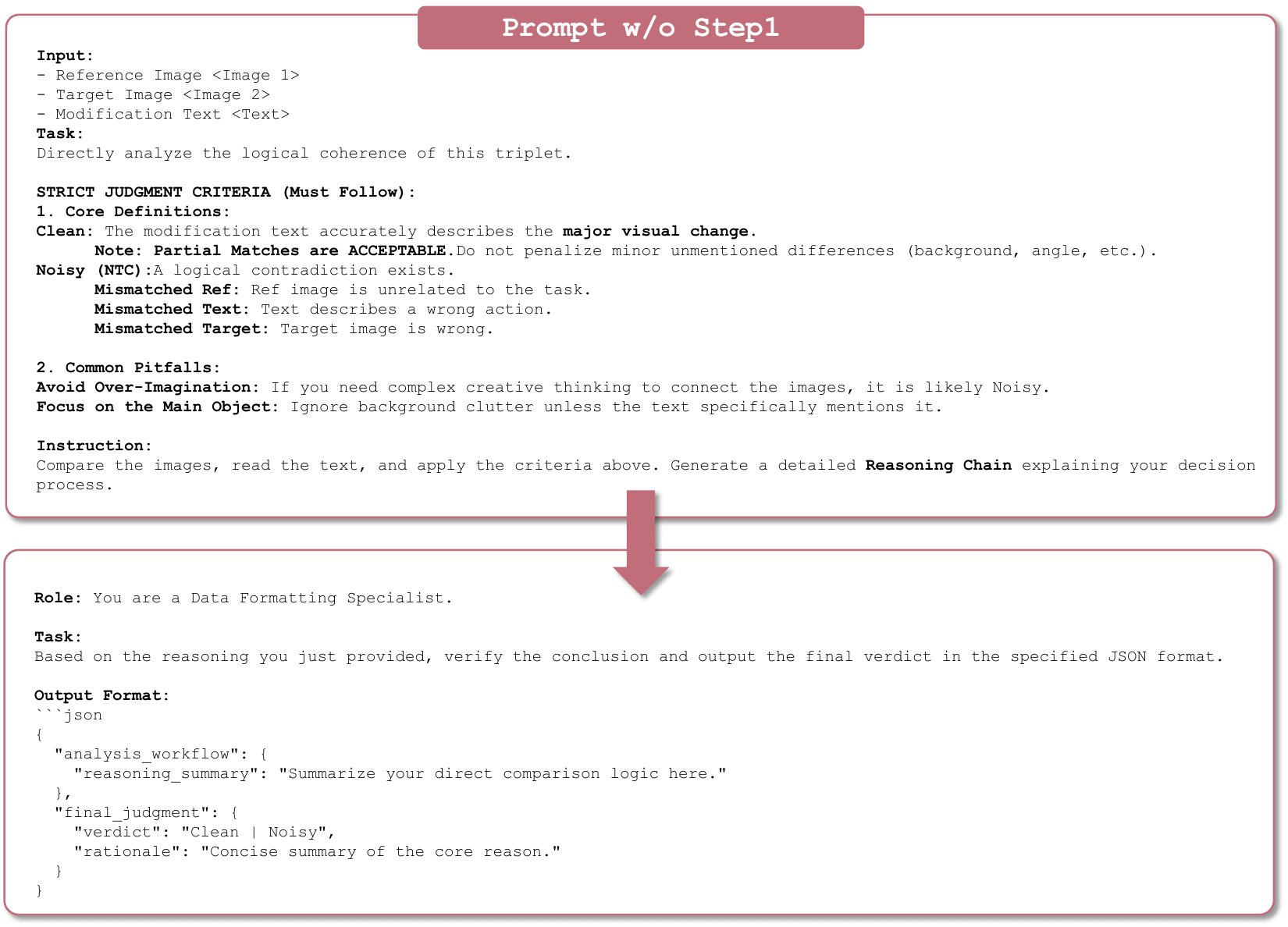}}
    \caption{The prompt variant in which Step 1 (input deconstruction) is removed. In this setting, we eliminated the instruction requiring the model to independently parse the image and text content, thereby compelling the model to execute the comparison directly without establishing objective factual anchors.}
    \label{fig:prompt2}
\end{figure*}

As illustrated in Figure~\ref{fig:ntc}, we visualize the sample discrimination results of Air-Know in NTC scenarios and the cleanliness estimation output by the \textit{Expert-Knowledge Internalization} (EKI) module. The visualization results intuitively confirm that, benefiting from the internalization of expert discriminative logic by EKI, the model accurately distinguishes between logically consistent clean samples and NTC noise containing semantic conflicts. Consequently, it outputs highly credible confidence scores, exhibiting superior noise-robust discriminative capabilities.

Specifically, the blue region on the left illustrates the triplets identified as ``Clean'' by the model, which generally achieved exceptionally high confidence scores. Taking the second row on the left as an example, when presenting the subtle spatial positional and attribute transformations between the reference image and the target image (specifically, the modification text describing ``Orange pouch is in front of blue pouch''), the model assigned a high score of $0.96$. This demonstrates that Air-Know accurately captures fine-grained semantic consistency among multimodal inputs rather than merely relying on superficial visual similarity, thereby validating its profound understanding of effective semantic variations.

In contrast, the orange region on the right displays typical NTC noisy correspondence, to which the model assigns significantly low confidence. For instance, in the middle case of the first row, although the modification text describes a certain renovation change, a fundamental Cross-category Semantic Mismatch exists between the reference image (two pugs) and the target image (shelves). This severe logical disconnection is keenly identified by the model, resulting in an extremely low score of 0.17. Similarly, in the case on the right of the second row, the modification text instruction ``Show shorn dog'' is completely disconnected from the actual visual content presented (bottles and vases). The model assigns a low score of 0.19 to this, demonstrating its high sensitivity to the consistency between the modification text and the target image.

\section{Prompt}
\label{sup:prompt}
We provide a detailed visualization of the specific prompt design logic utilized in the experiments in Figures \ref{fig:prompt1}, \ref{fig:prompt2}, \ref{fig:prompt3}, and \ref{fig:prompt4}.

\begin{figure*}[t]
    \centering
    \resizebox{0.9\linewidth}{!}{\includegraphics[width=\linewidth]{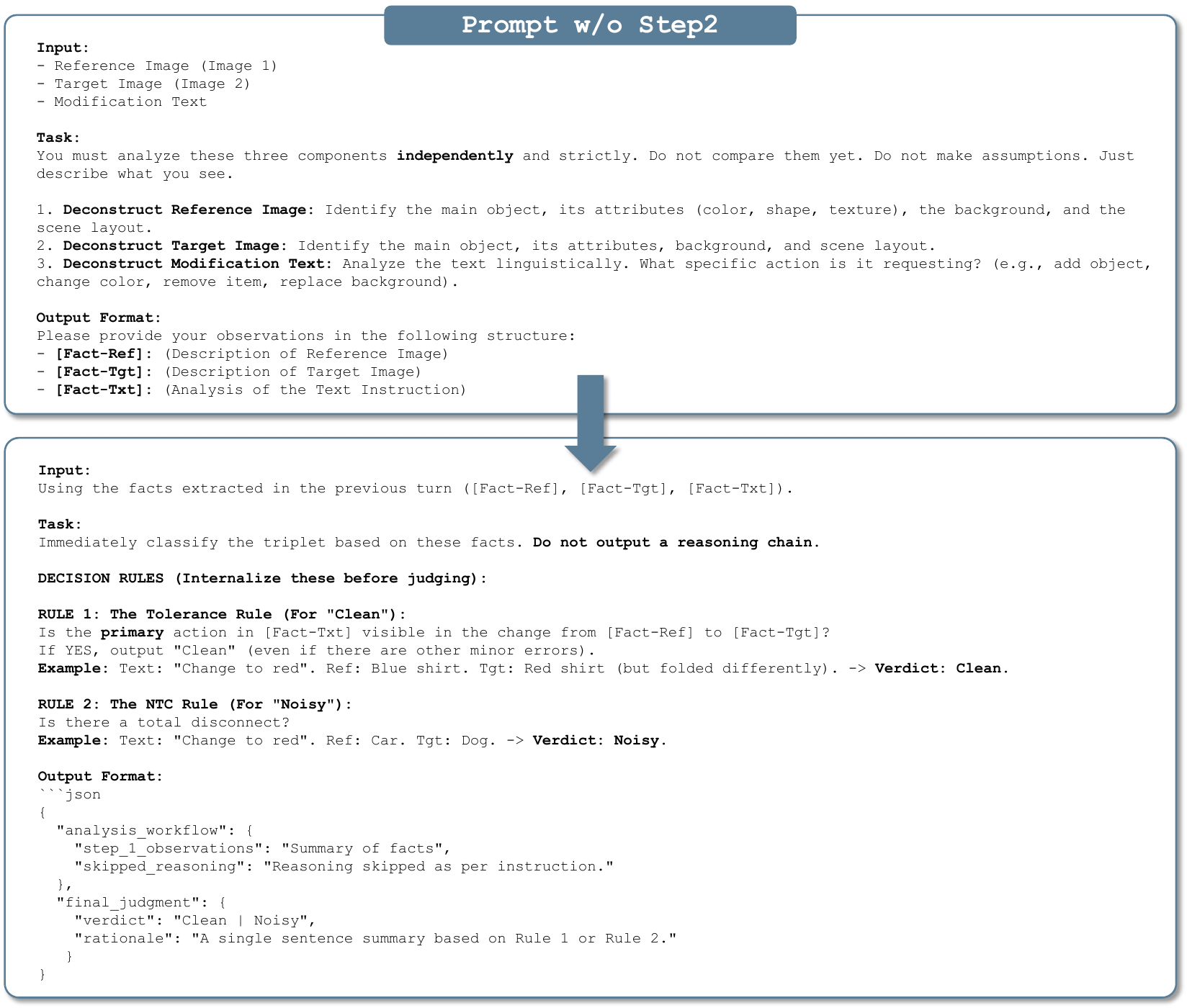}}
    \caption{The prompt variant in which Step 2 (Comparison and Reasoning) is removed. While this design retains the preliminary observation of the inputs, it omits the core reasoning chain involving the inference of the actual visual change ($\Delta T_I$) and the execution of cross validation. The model is required to proceed directly from observation to conclusion.}
    \label{fig:prompt3}
\end{figure*}

\begin{figure*}[t]
    \centering
    \resizebox{0.9\linewidth}{!}{\includegraphics[width=\linewidth]{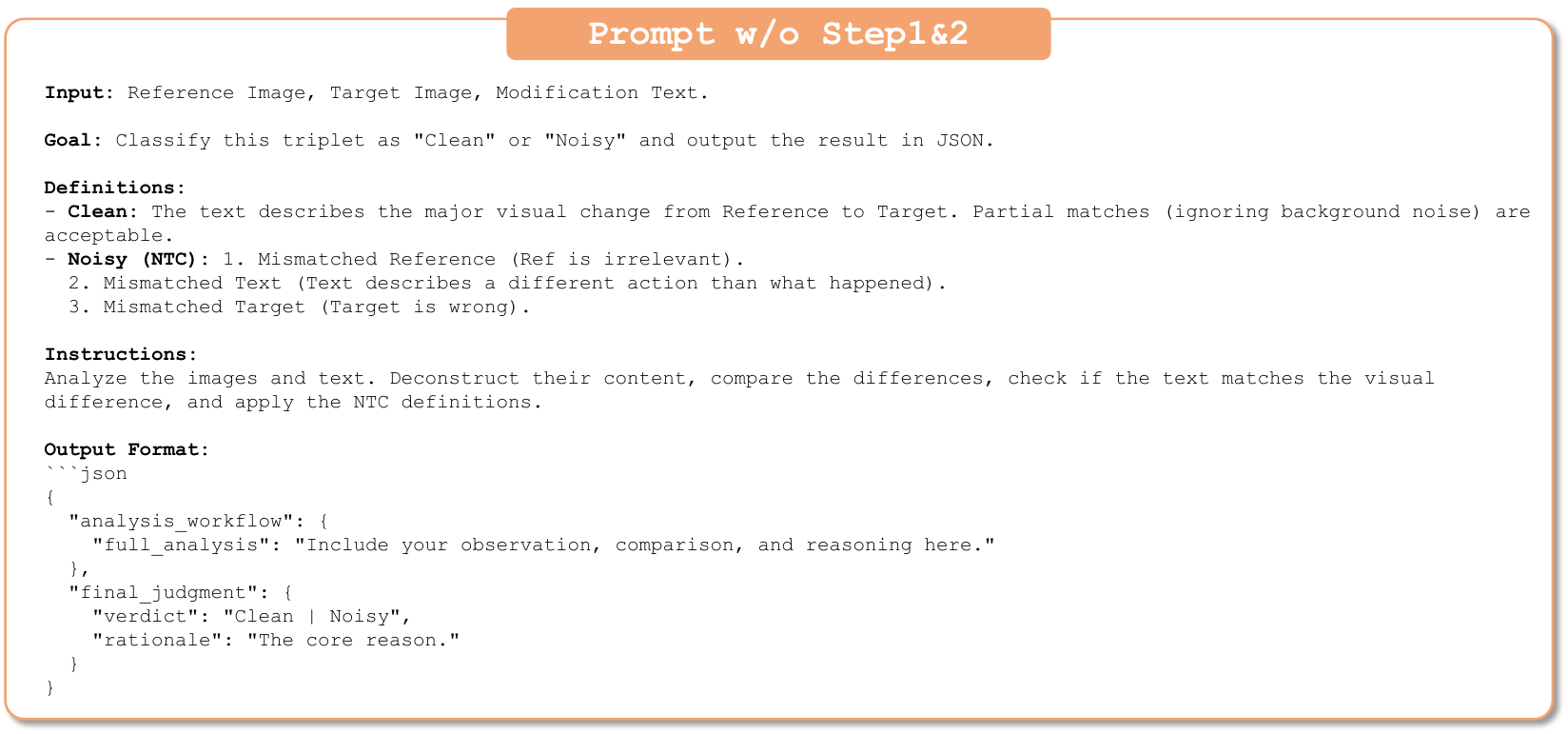}}
    \caption{The end-to-end prompt variant in which both Step 1 and Step 2 are removed. This variant strips away all structured intermediate steps, requiring the model to directly output a binary classification result (Clean/Noisy) based on the raw triplet input.}
    \label{fig:prompt4}
        \vspace{-10pt}
\end{figure*}

\begin{figure*}[t]
    \centering
    \resizebox{0.8\linewidth}{!}{\includegraphics[width=\linewidth]{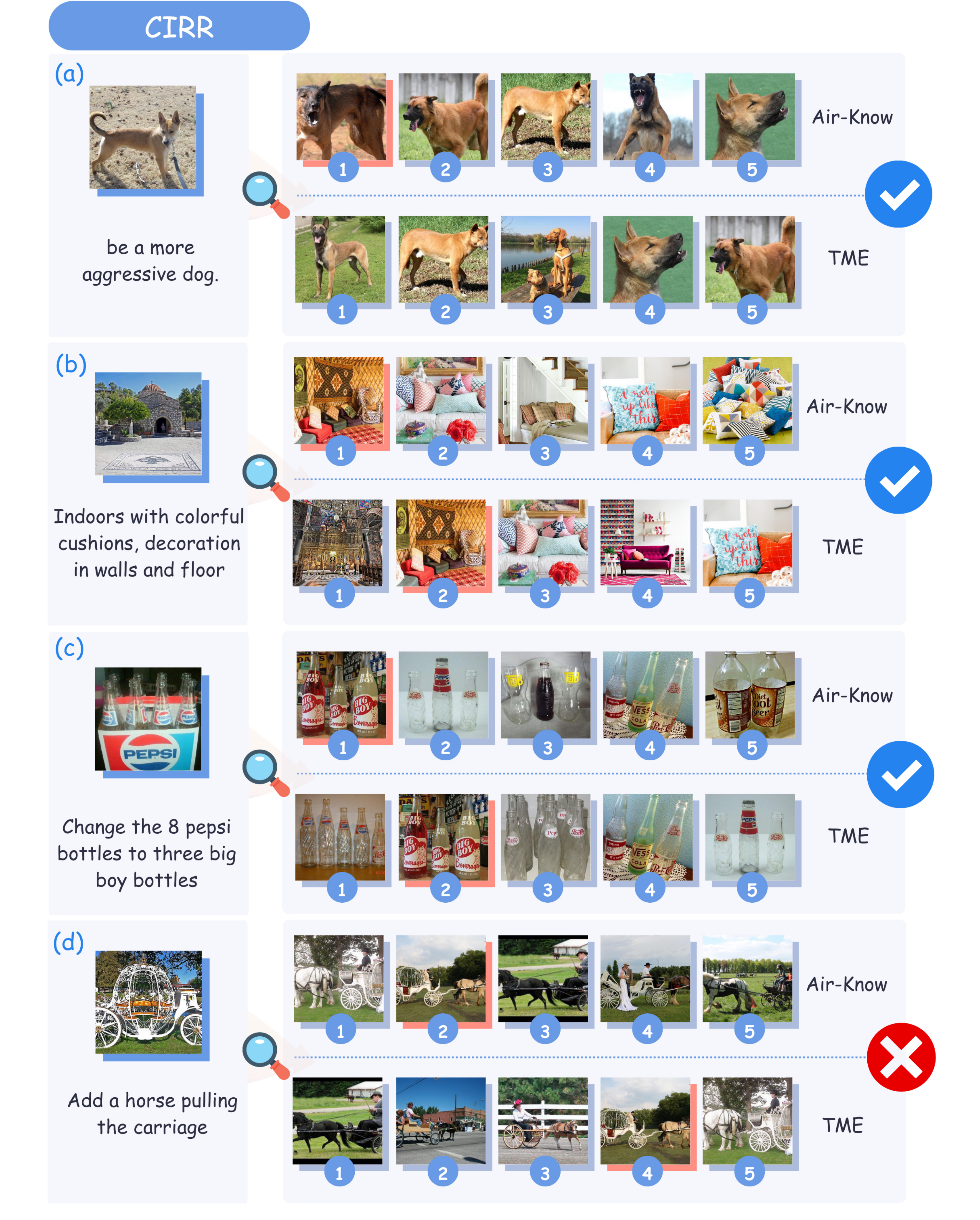}}
    \caption{Qualitative comparison on CIRR. We visualize the retrieval results to demonstrate the performance of the model under different query conditions. The target image is marked with orange shading. In (a-c), Air-Know successfully retrieved the target image at the Top-1 position. And (d) displays a failure case where the target image is not matched.}
    \label{fig:more_case_cirr}
\end{figure*}

\begin{figure*}[t]
    \centering
    \resizebox{0.8\linewidth}{!}{\includegraphics[width=\linewidth]{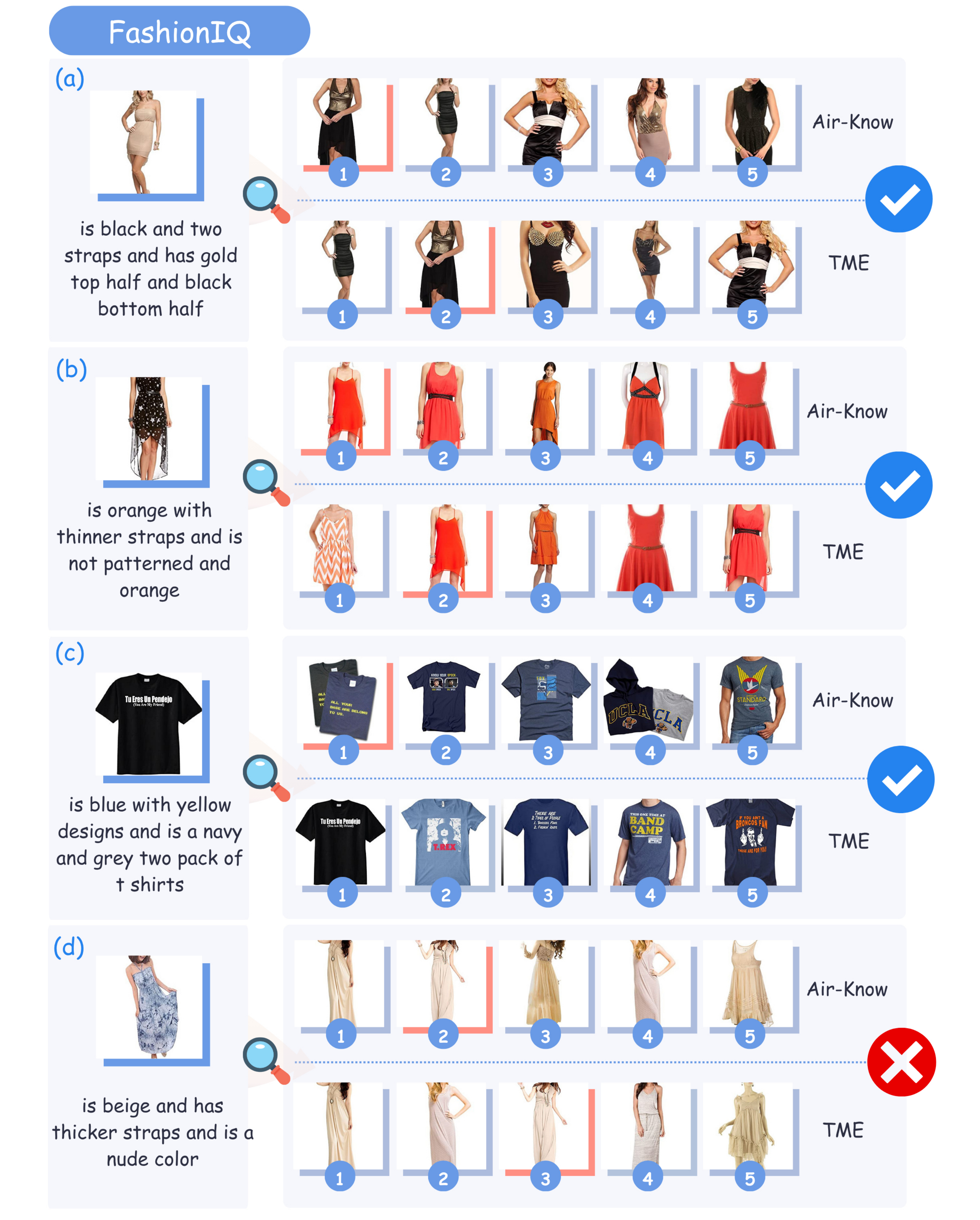}}
    \caption{Qualitative comparison on FashionIQ. We compare Air-Know with TME to demonstrate the retrieval results in scenarios involving fine-grained attribute modification, where the target image is marked with orange shading. In (a)-(c), Air-Know successfully retrieves the target image at the Top-1 position. And (d) presents a failure case where the target image is not matched.}
    \label{fig:more_case_fiq}
\end{figure*}

\clearpage
{
    \small
    \bibliographystyle{ieeenat_fullname}
    \bibliography{main}
}
% \end{CJK}
\end{document}